\documentclass{article}

\usepackage[utf8]{inputenc} % allow utf-8 input
\usepackage[T1]{fontenc}    % use 8-bit T1 fonts
\usepackage[colorlinks=true,linkcolor=black,anchorcolor=black,citecolor=black,filecolor=black,menucolor=black,runcolor=black,urlcolor=black]{hyperref}       % hyperlinks
\usepackage{url}            % simple URL typesetting
\usepackage{booktabs}       % professional-quality tables
\usepackage{amsfonts}       % blackboard math symbols
\usepackage{nicefrac}       % compact symbols for 1/2, etc.
\usepackage{microtype}     % microtypography
\usepackage{xcolor}         % colors
\usepackage{bm}         % italic-bold
\usepackage{graphicx}  		% figures
\usepackage{amsmath} % eqref
\usepackage{cite} % cite
\usepackage{diagbox}
\usepackage{siunitx}
\usepackage{multirow}
\usepackage{tabularx}
\usepackage{hhline}
\usepackage{arydshln}		% dotted lines in a table
\usepackage{xcolor}
\usepackage{mathtools} % coloneqq
\usepackage{amsthm}			% theorems
\usepackage{amssymb}			% varnothing
\usepackage{algorithm}
\usepackage{pifont}
\usepackage{threeparttable}
\usepackage{fullpage}

\usepackage{algpseudocode}	
\usepackage{arydshln}
\usepackage[export]{adjustbox}
\usepackage{booktabs}

\usepackage{mathtools} 
\usepackage{microtype} 
\usepackage{multirow}
\usepackage{nicefrac}       
\usepackage{pifont}   

\usepackage{url}
\usepackage{siunitx}
\usepackage{tabularx}
\usepackage{threeparttable}
\usepackage{xcolor}  

\definecolor{darkblue}{rgb}{0,0 ,0.542}
\definecolor{lightgreen}{rgb}{.9,1,.9}
\definecolor{lightred}{rgb}{1,.415,.415}
\definecolor{lightblue}{rgb}{.415,.415,1}

\newcolumntype{L}[1]{>{\raggedright\arraybackslash}p{#1}}
\newcolumntype{C}[1]{>{\centering\arraybackslash}p{#1}}
\newcolumntype{R}[1]{>{\raggedleft\arraybackslash}p{#1}}

\theoremstyle{plain} % plain = italic, definition = roman

\newtheorem{theorem}{Theorem}
\newtheorem{lemma}{Lemma}

\newtheorem{assumption}{Assumption}

\newtheorem*{suptheorem}{Theorem}

\def\defn{\,\coloneqq\,}

\def\d{{\mathsf{\, d}}}

\def\Im{{\mathsf{Im}}}
\def\prox{{\mathsf{prox}}}

\def\max{{\mathsf{max}}}
\def\min{\mathop{\mathsf{min}}}

\def\R{\mathbb{R}}
\def\E{\mathbb{E}}

\def\ebm{{\bm{e}}}

\def\xbm{{\bm{x}}}
\def\zbm{{\bm{z}}}
\def\ybm{{\bm{y}}}
\def\zbm{{\bm{z}}}
\def\sbm{{\bm{s}}}
\def\abm{{\bm{a}}}
\def\bbm{{\bm{b}}}

\def\ubm{{\bm{u}}}
\def\vbm{{\bm{v}}}

\def\varepsilonbar{{\overline{\varepsilon}}}

\def\Abm{{\bm{A}}}
\def\Dbm{{\bm{D}}}

\def\Ibm{{\bm{I}}}
\def\Sbm{{\bm{S}}}

\def\Ibf{{\mathbf{I}}}

\def\Abm{{\bm{A}}}
\def\Dbm{{\bm{D}}}

\def\Ibm{{\bm{I}}}
\def\Hbm{{\bm{H}}}

\def\Ncal{{\mathcal{N}}}

\def\Lcal{{\mathcal{L}}}

\def\Dsf{{\mathsf{D}}}

\def\Tsf{{\mathsf{T}}}

\def\Tsf{{\mathsf{T}}}

\def\Dsf{{\mathsf{D}}}

\def\Hsf{{\mathsf{H}}}
\def\Jsf{{\mathsf{J}}}

\def\Hsf{{\mathsf{H}}}

\def\zbmbar{{\overline{\bm{z}}}}
\def\xbmhat{{\widehat{\bm{x}}}}

\def\hine{{\hat{h}}}

\def\Dhat{{\mathsf{\widehat{D}}}}

\def\argmin{\mathop{\mathsf{arg\,min}}} % Argument of a minimization

\title{Prior Mismatch and Adaptation in PnP-ADMM with a Nonconvex Convergence Analysis}
\date{}

\author{
Shirin Shoushtari\textsuperscript{*}, Jiaming Liu\textsuperscript{*}, Edward P.\ Chandler,\\ M.\ Salman Asif, and Ulugbek S.\ Kamilov\\
\small Washington University in St. Louis, MO 63130, USA\\
\small University of California Riverside, CA 92521, USA\\
\small \texttt{\{s.shirin,  jiaming.liu, e.p.chandler, kamilov\}@wustl.edu, \texttt{sasif@ucr.edu} }
}

\begin{document}

\maketitle
\let\thefootnote\relax\footnote{\textsuperscript{*}These authors contributed equally.}
\begin{abstract}
\noindent
Plug-and-Play (PnP) priors is a widely-used family of methods for solving imaging inverse problems by integrating physical measurement models with image priors specified using image denoisers.  PnP methods have been shown to achieve state-of-the-art performance when the prior is obtained using powerful deep denoisers. Despite extensive work on PnP, the topic of \emph{distribution mismatch} between the training and testing data has often been overlooked in the PnP literature. This paper presents a set of new theoretical and numerical results on the topic of prior distribution mismatch and domain adaptation for \emph{alternating direction method of multipliers (ADMM)} variant of PnP. Our theoretical result provides an explicit error bound for PnP-ADMM due to the mismatch between the desired denoiser and the one used for inference. Our analysis contributes to the work in the area by considering the mismatch under \emph{nonconvex} data-fidelity terms and \emph{expansive} denoisers. Our first set of numerical results quantifies the impact of the prior distribution mismatch on the performance of PnP-ADMM on the problem of image super-resolution. Our second set of numerical results considers a simple and effective domain adaption strategy that closes the performance gap due to the use of mismatched denoisers. Our results suggest the relative robustness of PnP-ADMM to prior distribution mismatch, while also showing that the performance gap can be significantly reduced with few training samples from the desired distribution.
\end{abstract}

\section{Introduction}
\label{sec:intro}

\emph{Imaging inverse problems} consider  the recovery of a clean image from its corrupted observation. Such problems arise across the fields of computational imaging, biomedical imaging, and computer vision. As 
imaging inverse problems are typically ill-posed, solving them requires the use of image priors. While many approaches have been proposed for implementing image priors, the current literature is primarily focused on methods based on training \emph{deep learning (DL)} models to map noisy observations to clean images~\cite{McCann.etal2017, Lucas.etal2018, Ongie.etal2020}.

\medskip\noindent
\emph{Plug-and-Play (PnP)} Priors~\cite{Venkatakrishnan.etal2013, Sreehari.etal2016} has emerged as a class of DL algorithms for solving inverse problems by denoisers as image priors. PnP has been successfully used in many applications such as super-resolution, phase retrieval, microscopy, and medical imaging~\cite{Metzler.etal2018, Zhang.etal2017a, Meinhardt.etal2017, Dong.etal2019, Zhang.etal2019, Wei.etal2020, Zhang.etal2021dpir}. The success of PnP has resulted in the development of its multiple variants (e.g., PnP-PGM, PnP-SGD, PnP-ADMM. PnP-HQS), strong interest in its theoretical analysis, as well as investigation of its connection to other methods used in inverse problems, such as score matching and denoising diffusion probabilistic models~\cite{Chan.etal2016, Romano.etal2017, Buzzard.etal2017, Teodoro.etal2019, Ahmad.etal2020, yuan2020plug, Kamilov.etal2022, Reehorst.Schniter2019, Sun.etal2019a, Sun2019b, Liu.etal2021b, kadkhodaie2021stochastic, Cohen.etal2021, Hurault.etal2022, Laumont.etal2022}.

\medskip\noindent
Despite extensive literature on PnP, the research in the area has mainly focused on the setting where the distribution of the test or inference data is perfectly matched to that of the data used for training the image denoiser. Little work exists for PnP under mismatched priors, where a distribution shift exists between the training and test data. In this paper, we investigate the problem of \emph{mismatched} priors in PnP-ADMM. We present a new theoretical analysis of PnP-ADMM that accounts for the use of mismatched priors. Unlike most existing work on PnP-ADMM, our theory is compatible with \emph{nonconvex} data-fidelity terms and \emph{expansive} denoisers~\cite{Sun.etal2021, Tang.Davies2020, gavaskar2019proof, Chan.etal2019, Ryu.etal2019}. Our analysis establishes explicit error bounds on the convergence of PnP-ADMM under a well-defined set of assumptions. We validate our theoretical findings by presenting numerical results on the influence of distribution shifts, where the denoiser trained on one dataset (e.g., BreCaHAD or CelebA) is used to recover an image from another dataset (e.g., MetFaces or RxRx1). We additionally present numerical results on a simple domain adaptation strategy for image denoisers that can effectively address data distribution shifts in PnP methods  (see Figure~\ref{fig:disvsgap} for an illustration). Our work thus enriches the current PnP literature by providing novel theoretical and empirical insights into the problem of data distribution shifts in PnP.

\medskip\noindent
All proofs and some details that have been omitted due to space constraints of the main text are included in the supplementary material.
%%%%%%%%%%%%%%%%%%
\begin{figure}[t!]
    \centering
     \includegraphics[width=0.95\textwidth,center]{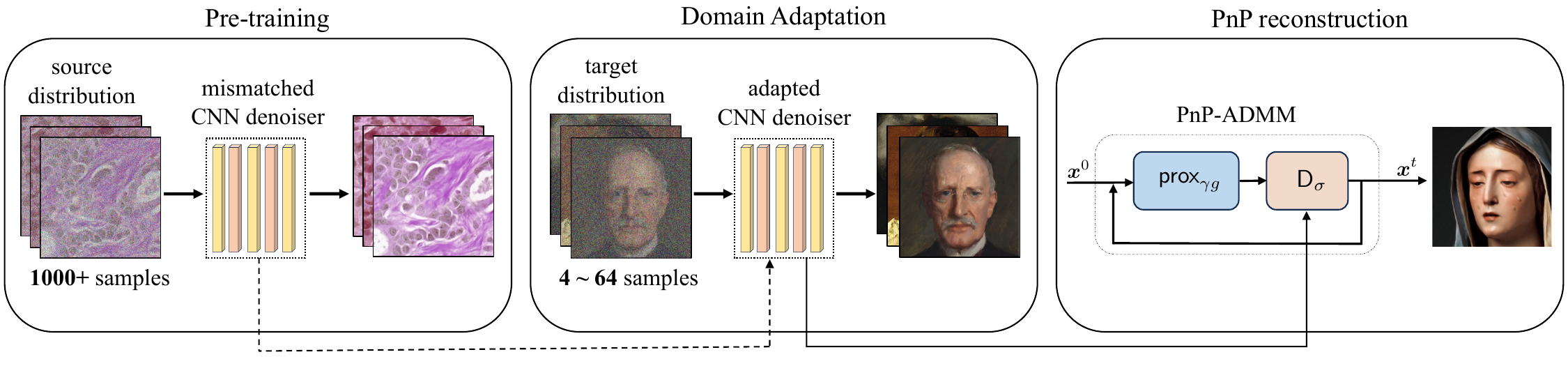}
    \caption{\emph{Illustration of domain adaptation in PnP-ADMM. The mismatched denoiser is pre-trained on source distribution (BreCaHAD) and adapted to target distribution (MetFaces) using a few samples. Adapted prior is then plugged into PnP-ADMM algorithm to reconstruct a sample from MetFaces.}
    }
    \label{fig:main}
\end{figure}
%%%%%%%%%%%%%%%%%%

\section{Background}
\label{sec:background}
\textbf{Inverse problems.} Inverse problems involve the recovery of an unknown signal $\xbm \in \R^n$ from a set of noisy measurements $\ybm = \Abm \xbm + \ebm $, where $\Abm \in \R^{m\times n}$ is the measurement model and $\ebm$ is the noise. Inverse problems are often formulated and solved as optimization problems of form
\begin{equation}
    \label{Eq:OptimInverseProb}
    \xbmhat \in \argmin_{\xbm\in\R^n} f(\xbm) \quad\text{with}\quad f(\xbm) = g(\xbm) + h(\xbm)\ ,
\end{equation}
where $g$ is the data-fidelity term that measures the consistency with the measurements $\ybm$ and $h$ is the regularizer that incorporates prior knowledge on $\xbm$. The least-squares function $g(\xbm) = \frac{1}{2} \|\Abm \xbm - \ybm\|_2^2$ and total variation (TV) function $h(\xbm) = \tau \|\Dbm \xbm\|_1$, where $\Dbm$ denotes the  image gradient and $ \tau >0$ a regularization parameter, are commonly used functions for the data-fidelity term and the regularizer~\cite{Rudin.etal1992, Beck.Teboulle2009a}.

\medskip\noindent
\textbf{Deep Learning.} DL has gained significant attention in the context of inverse problems~\cite{McCann.etal2017, Lucas.etal2018, Ongie.etal2020}. DL methods seek to perform a regularized inversion by learning a mapping from the measurements to the target images parameterized by a deep convolutional neural network (CNN)~\cite{Wang.etal2016, Jin.etal2017, Kang.etal2017, zafari2023frequency, Chen.etal2017, Xu.etal2018}. Model-based DL (MBDL) refers to a sub-class of DL methods for inverse problems that also integrate the measurement model as part of the deep model~\cite{Ongie.etal2020, Monga.etal2021}. MBDL approaches include methods such as PnP, regularization by denoising (RED), deep unfolding (DU), and deep equilibrium models (DEQ)~\cite{zhang2018ista, Hauptmann.etal2018, Gilton.etal2021, Liu.etal2022a}.

\medskip\noindent
\textbf{Plug-and-Play Priors.} 
PnP is one of the most popular MBDL approaches for solving imaging inverse problems that uses denoisers as priors~\cite{Venkatakrishnan.etal2013} (see also recent reviews~\cite{Ahmad.etal2020, Kamilov.etal2022}). 
PnP has been extensively investigated, leading to multiple PnP variants and theoretical analyses~\cite{Chan.etal2016, Buzzard.etal2017, Sun.etal2021, Ryu.etal2019, Hurault.etal2022, Laumont.etal2022, Tirer.Giryes2019, Teodoro.etal2019, Sun.etal2019a, Cohen.etal2021}. Existing theoretical convergence analyses of PnP differ in the specifics of the assumptions required to ensure the convergence of the corresponding iterations.
For example, bounded, averaged, firmly nonexpansive, nonexpansive, residual nonexpansive, or demi-contractive denoisers have been previously considered for designing convergent PnP schemes~\cite{Chan.etal2016,gavaskar2019proof, Romano.etal2017, Ryu.etal2019, Cohen.etal2021, Sun2019b, Sun.etal2021, Terris2020, Reehorst.Schniter2019, Liu.etal2021b, hertrich2021convolutional, bohra2021learning}.
The recent work~\cite{Xu.etal2020} has used an elegant formulation of an MMSE denoiser from~\cite{Gribonval2011} to perform a nonconvex convergence analysis of PnP-PGM without any nonexpansiveness assumptions on the denoiser. Another recent line of PnP work has explored specification of the denoiser as a gradient-descent step on a functional parameterized by a deep neural network~\cite{Hurault.etal2022, hurault2022proximal, Cohen.etal2021}.

\medskip\noindent
PnP-ADMM is summarized in Algorithm~\ref{Alg:InexactPnPADMM}~\cite{Sreehari.etal2016, Venkatakrishnan.etal2013}, where $
\Dsf_\sigma$ is an additive white Gaussian denoiser (AWGN) denoiser, $\gamma >0$ is the penalty parameter, and $\sigma > 0$ controls the denoiser strength. PnP-ADMM is based on the alternating direction method of multipliers (ADMM)~\cite{Boyd.etal2011}. Its formulation relies on optimizing in an alternating fashion the augmented Lagrangian associated with the objective function in~\eqref{Eq:OptimInverseProb}
 \begin{equation}
    \label{Eq:AugLagrangian}
    \phi \left (  \xbm, \zbm, \sbm \right) = g(\xbm) + h(\zbm) + \frac{1}{\gamma} \sbm^\Tsf (\xbm - \zbm) + \frac{1}{2\gamma} \left \| \xbm - \zbm \right  \|_2^2.
\end{equation}

\medskip\noindent
The theoretical convergence of PnP-ADMM has been explored for convex functions using monotone operator theory~\cite{Ryu.etal2019, Sun.etal2021}, for nonconvex regularizer and convex data-fidelity terms~\cite{hurault2022proximal}, and for bounded denoisers~\cite{Chan.etal2016}. 

\medskip\noindent
\textbf{Distribution Shift.} 
Distribution shifts naturally arise in imaging when a DL model trained on one type of data is applied to another. The mismatched DL models due to distribution shifts lead to suboptimal performance. Consequently, there has been interest in mitigating the effect of mismatched DL models~\cite{sun2020test, Darestani.etal2021, Darestani.etal2022, Jalal.etal2021a}. In PnP methods, a mismatch arises when the denoiser is trained on a distribution different from that of the test data. The prior work on denoiser mismatch in PnP is limited~\cite{shoushtari2022deep, Reehorst.Schniter2019, Laumont.etal2022, Liu.etal2020}.

\medskip\noindent
\textbf{Our contributions.}
\emph{\textbf{(1)}} Our first contribution is a new theoretical analysis of PnP-ADMM accounting for the discrepancy between the desired and mismatched denoisers. Such analysis has not been considered in the prior work on PnP-ADMM. Our analysis is broadly applicable in the sense that it does \emph{not} assume convex data-fidelity terms and nonexpansive denoisers. \emph{\textbf{(2)}} Our second contribution is a comprehensive numerical study of distribution shifts in PnP through several well-known image datasets on the problem of image super-resolution. \emph{\textbf{(3)}} Our third contribution is the illustration of simple data adaptation for addressing the problem of distribution shifts in PnP-ADMM. We show that one can successfully close the performance gap in PnP-ADMM due to distribution shifts by adapting the denoiser to the target distribution using very few samples. 

\section{Proposed Work}
\label{sec:theoanalysis}

This section presents the convergence analysis of PnP-ADMM that accounts for the use of mismatched denoisers. It is worth noting that the theoretical analysis of PnP-ADMM has been previously discussed in~\cite{Chan.etal2019, Teodoro.etal2019, Ryu.etal2019, Sun.etal2021}. The novelty of our work can be summarized in two aspects: (1) we analyze convergence with the mismatched priors; (2) our theory accommodates nonconvex $g$ and expansive denoisers. 

\subsection{PnP-ADMM with Mismatched Denoiser}\label{subsec:algorithm}

\begin{algorithm}[t]
\caption{PnP-ADMM}
\label{Alg:InexactPnPADMM}
\begin{algorithmic}[1]
\State \textbf{input: } $\zbm^0, \sbm^0 \in \R^n$, parameters $\sigma, \gamma>0$.
\For{$k = 1, 2, 3, \cdots$}
\State $\xbm^k \leftarrow \prox_{\gamma g} (\zbm^{k-1} - \sbm^{k-1}) $
\State $\zbm^k\leftarrow \Dsf_{\sigma} \left (\xbm^{k} + \sbm^{k-1} \right )$ 
% \Comment{$\Dhat_\sigma$ is the \textbf{mismatched} prior.}
\State $\sbm^{k} \leftarrow \sbm^{k-1} + \xbm^k - \zbm^k $
\EndFor
\end{algorithmic}
\end{algorithm}

We denote the target distribution as $p_\xbm$ and the mismatched distribution as $\widehat{p}_\xbm$. The mismatched denoiser $\Dhat_\sigma$ is a \emph{minimum mean squared error (MMSE)} estimator for the AWGN denoising problem
\begin{equation}\label{Eq:noisemodel}
    \vbm = \xbm +\ebm \quad \text{with} \quad \xbm \sim \widehat{p}_\xbm, \quad \ebm \sim \Ncal(0, \sigma^2 \Ibm).
\end{equation}
The MMSE denoiser is the conditional mean estimator for~\eqref{Eq:noisemodel} and can be expressed as
\begin{equation}
\label{Eq:MMSEDenoiser}
\Dhat_{\sigma}(\vbm) \defn \E[\xbm | \vbm] = \int_{\R^{n}} \xbm \widehat{p}_{\xbm | \vbm} (\xbm | \vbm) \d \xbm,
\end{equation}
where $\widehat{p}_{\xbm | \vbm} (\xbm | \vbm) \propto G_\sigma(\vbm-\xbm) \widehat{p}_\xbm(\xbm)$, with $G_\sigma$ denoting the Gaussian density.
We refer to the MMSE estimator $\Dhat_\sigma$,  corresponding to the mismatched data distribution $\widehat{p}_\xbm$, as the mismatched prior. 

\medskip\noindent
Since the integral~\eqref{Eq:MMSEDenoiser} is generally intractable, in practice, the denoiser corresponds to a deep model trained to minimize the mean squared error (MSE) loss  
\begin{equation}
\label{Eq:MSELoss}
\Lcal(\Dhat_\sigma) = \E \left[ \|\xbm - \Dhat_{\sigma}(\vbm)\|_2^2 \right].
\end{equation}
MMSE denoisers trained using the MSE loss are optimal with respect to the widely used image-quality metrics in denoising, such as signal-to-noise ratio (SNR), and have been extensively used in the PnP literature~\cite{Xu.etal2020, Laumont.etal2022, bigdeli2017, kadkhodaie2021stochastic, gan2023block}.

\medskip\noindent
When using a mismatched prior in PnP-ADMM, we replace Step 4 in Algorithm~\ref{Alg:InexactPnPADMM} by 
\begin{equation}
    \zbm^k\leftarrow \Dhat_{\sigma} \left (\xbm^{k} + \sbm^{k-1} \right ), 
\end{equation}
where $\Dhat_\sigma$ is the mismatched MMSE denoiser. To avoid confusion, we denote by $\zbm^k$ and $\zbmbar^k$ the outputs of the mismatched and target denoisers at the $k$ iteration, respectively. Consequently, we have 
$\zbmbar^k = \Dsf_\sigma (\xbm^{k} + \sbm^{k-1}),$ where $\Dsf_\sigma$ is the target MMSE denoiser.

\subsection{Theoretical Analysis}\label{subsec:theory}

\medskip\noindent 
Our analysis relies on the following set of assumptions that serve as sufficient conditions.  
\begin{assumption}
\label{As:NonDegenerate}
The prior distributions  $p_{\xbm}$ and  $\widehat{p}_{\xbm}$, denoted as target and mismatched priors respectively, are non-degenerate over $\R^{n}$.
\end{assumption}

\medskip\noindent
A distribution is considered degenerate over $\R^n$ if its support is confined to a lower-dimensional manifold than the dimensionality of $n$. Assumption~\ref{As:NonDegenerate} is useful to establish an explicit link between a MMSE denoiser and its associated regularizer. For example, the regularizer $h$ associated with the target MMSE denoiser $\Dsf_\sigma$ can be expressed as (see~\cite{Gribonval2011, Xu.etal2020} for background)
\begin{equation}
\label{Eq:ExpReg}
h (\xbm) \defn
\begin{cases}
-\frac{1}{2\gamma}\|\xbm - \Dsf_{\sigma}^{-1}(\xbm)\|_2^2 + \frac{\sigma^2}{\gamma} h_{\sigma}(\Dsf_{\sigma}^{-1}(\xbm)) & \text{for } \xbm \in \Im(\Dsf_{\sigma}) \\
+\infty & \text{for } \xbm \notin \Im(\Dsf_{\sigma}),
\end{cases}
\end{equation}
where $\gamma > 0$ denotes the penalty parameter, $\Dsf_{\sigma}^{-1}: \Im(\Dsf_{\sigma}) \rightarrow \R^{n}$ represent a well defined and smooth inverse mapping over $\Im(\Dsf_{\sigma})$, and $h_{\sigma}(\cdot) \defn -\log(p_{\ubm}(\cdot))$, with $p_{\ubm}$ denoting the probability distribution over the AWGN corrupted observations
$$
    \ubm = \xbm +\ebm \quad \text{with} \quad \xbm \sim p_\xbm, \quad \ebm \sim \Ncal(0, \sigma^2 \Ibm),  
$$
(the derivation is provided in Section~\ref{ssec:MMSEdenoiseProx} for completeness). Note that the smoothness of both  $\Dsf_{\sigma}^{-1}$ and $h_{\sigma}$ guarantees the smoothness of the function $h$. Additionally,  similar connection exist between the mismatched MMSE denoiser $\Dhat_\sigma$ and the regularizer $\hine(\xbm)$, with $\hine_{\sigma}(\cdot) \defn -\log(\widehat{p}_{\vbm}(\cdot))$ characterizing  the relationship between mismatched denoiser and shifted distribution.
\begin{assumption}
\label{As:LipschitzDataFit}
The function $g$ is continuously differentiable.
\end{assumption}

\medskip\noindent
This  assumption is a standard assumption used in nonconvex optimization, specifically in the context of inverse problems~\cite{li2018simple, jiang2019structured,yashtini2021multi}.

\begin{assumption}
\label{As:BoundedFromBelow}
The data-fidelity term and the implicit regularizers are bounded from below.
\end{assumption}

\medskip\noindent
Assumption~\ref{As:BoundedFromBelow} implies that there exists $f^\ast > -\infty$ such that $f(\xbm) \geq f^\ast$ for all $\xbm \in \R^n$.

\begin{assumption}
\label{As:LipschitzPrior} 
The denoisers $\Dsf_\sigma$ and $\Dhat_\sigma$ have the same range $\Im(\Dsf_\sigma)$.
Additionally, functions $h$ and $\hine$ associated with $\Dsf_\sigma$ and $\Dhat_\sigma$, are continuously differentiable with $L$-Lipschitz continuous gradients over $\Im(\Dsf_\sigma).$
\end{assumption}

\medskip\noindent
It is known (see \cite{Gribonval2011, Xu.etal2020}) that functions $h$ and $\hine$ are infinitely differentiable over their ranges. The assumption that the two image denoisers have the same range is also a relatively mild assumption. Ideally, both denoisers would have the same range corresponding to the set of desired images. Assumption~\ref{As:LipschitzPrior} is thus a mild extension that further requires Lipschitz continuity of the gradient over the range of denoisers.

\begin{assumption}
\label{As:InexactDistance}
The mismatched denoiser $\Dhat_{\sigma}$ satisfies
\begin{equation*}
\|\Dhat_{\sigma}(\vbm^k)-\Dsf_{\sigma}(\vbm^k)\|_2 \leq \delta_k, \quad k = 1, 2, 3, \ldots
\end{equation*}
where $\Dhat_{\sigma}$ is given in~\eqref{Eq:MMSEDenoiser} and $\vbm^k = \xbm^{k}+\sbm^{k-1}$ in Algorithm~\ref{Alg:InexactPnPADMM}.
\end{assumption}

\medskip\noindent
Our analysis assumes that at every iteration,  PnP-ADMM uses a mismatched MMSE denoiser, derived from a  shifted distribution. We consider the case where at iteration $k$ of PnP-ADMM, the distance of the outputs of $\Dsf_{\sigma}$ and $\Dhat_{\sigma}$ is bounded by a constant $\delta_k$.

\begin{assumption}
    \label{As:BoundIters}
    For the sequence $\{\xbm^k, \zbm^k, \sbm^k\}$ generated by iterations of PnP-ADMM with mismatched MMSE denoiser in Algorithm~\ref{Alg:InexactPnPADMM}, there exists a constant $R$ such that 
    \begin{equation*}
       \left \|\zbm^k - \zbm^*\right\|_2 \leq R,  \quad k =1, 2, 3, \ldots  
    \end{equation*}
    where $\left(\xbm^\ast, \zbm^\ast, \sbm^\ast\right)$ are stationary points of augmented Lagrangian in~\eqref{Eq:AugLagrangian}. 
\end{assumption}

\medskip\noindent
This assumption is a reasonable assumption since many images have bounded pixel values, for example $[0,255]$ or $[0,1]$.

\medskip\noindent
We are now ready to present our convergence result under mismatched MMSE denoisers.

\begin{theorem}\label{Thm:MainThm}
Run PnP-ADMM using a \textbf{mismatched} MMSE denoiser for $t \geq 1$ iterations under Assumptions~\ref{As:NonDegenerate}-\ref{As:BoundIters} with the penalty parameter  $0<\gamma\leq 1/(4L)$. Then, we have 
    \begin{equation*}
              \min_{1\leq k\leq t} \left \|\nabla f\left (\xbm^k\right )\right \|_2^2 \leq  \frac{1}{t} \sum_{k=1}^t\left\|\nabla f\left (\xbm^k\right )\right\|^2_2 
    \leq \frac{A_1}{t} \left(\phi\left (\xbm^0, \zbm^0, \sbm^0\right)-\phi^\ast\right) +  A_2 \varepsilonbar_t
    \end{equation*}
where $A_1>0$ and $A_2 >0$ are iteration independent constants, $\phi^* = \phi(\xbm^*, \zbm^*, \sbm^*)$, and $\varepsilonbar_t \defn (1/t)\left(\varepsilon_1 + \cdots + \varepsilon_t\right)$ is the error term that is an average of the quantities $\varepsilon_k \defn \max\{\delta_k, \delta_k^2\}$. In addition, if the sequence $\{\delta_i\}_{i\geq 1}$ is summable, we have that  $\|\nabla f(\xbm^t)\|_2 \to 0$  as $t \to \infty$.
\end{theorem}

\medskip\noindent
When we replace the mismatched MMSE denoiser with the \emph{target} MMSE denoiser, we recover the traditional PnP-ADMM. To highlight the impact of the mismatch, we next provide the same statement but using the target denoiser.
\begin{theorem}\label{Thm:ExactPnPADMM}
    Run PnP-ADMM with the MMSE denoiser for $t \geq 1$ iterations under Assumptions~\ref{As:NonDegenerate}-\ref{As:LipschitzPrior} with the penalty parameter  $0<\gamma\leq 1/(4L)$. Then, we have 
    \begin{equation*}
       \min_{1\leq k\leq t} \left \|\nabla f\left (\xbm^k\right )\right \|_2^2 \leq \frac{1}{t} \sum_{k=1}^t \left\|\nabla f(\xbm^k)\right \|^2_2 
    \leq \frac{C}{t} \left(\phi(\xbm^0, \zbm^0, \sbm^0)-\phi^\ast\right) , 
    \end{equation*}
where $C>0$ is a constant independent of iteration.
\end{theorem}

\medskip\noindent
The proof of Theorem~\ref{Thm:MainThm} is provided in the appendix. For completeness, we also provide the proof of Theorem~\ref{Thm:ExactPnPADMM}. Theorem~\ref{Thm:MainThm} provides insight into the convergence of PnP-ADMM using mismatched MMSE denoisers. It states that if $\delta_k$ are summable, the iterates of PnP-ADMM with mismatched denoisers satisfy $\nabla f(\xbm^t) \to \bm{0} $ as $t \to 0$ and PnP-ADMM converges to a stationary point of the objective function $f$ associated with the target denoiser. On the other hand, if the sequence $\delta_k$ is not summable, the convergence is only up to an error term that depends on the distance between the target and mismatched denoisers. Theorem~\ref{Thm:MainThm} can be viewed as a more flexible alternative for the convergence analyses in~\cite{Sun.etal2021, Chan.etal2019, Ryu.etal2019}. While the analyses in the prior works assume convex $g$ and nonexpansive residual, nonexpansive or bounded  denoisers, our analysis considers that denoiser  $\Dsf_\sigma$ is a mismatched MMSE estimator, where the mismatched denoiser distance to the target denoiser is bounded by $\delta_k$ at each iteration of PnP-ADMM. 

%%%%%%%%%%%%%%%%%
\begin{figure}[!t]
    \centering
\includegraphics[width=0.99\textwidth]{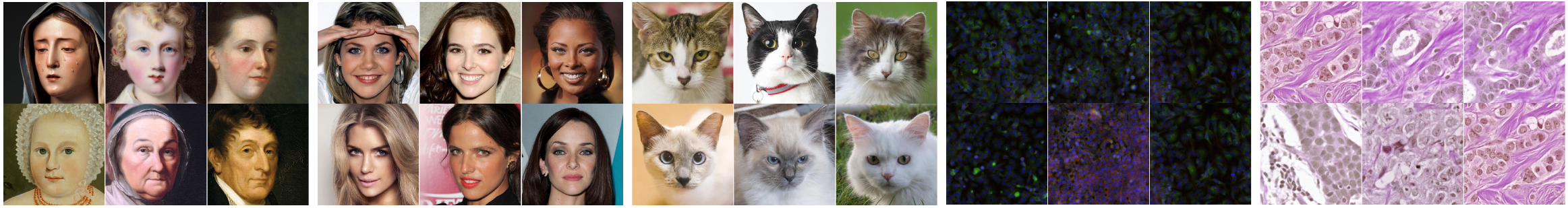}
    \caption{\emph{Sample images from the datasets used for training the denoisers. From left to right: MetFaces~\cite{karras2020training}, CelebA~\cite{liu2015faceattributes}, AFHQ~\cite{choi2020stargan}, RxRx1~\cite{sypetkowski2023rxrx1}, and ~BreCaHAD~\cite{aksac2019brecahad}. }}
    \label{fig:gts}
\end{figure}
%%%%%%%%%%%%%%%%%%

\medskip\noindent
 In conclusion, PnP-ADMM using a mismatched MMSE denoiser approximates the solution obtained by PnP-ADMM using the target MMSE denoiser with an error that depends on the discrepancy between the denoisers. Therefore, one can control the accuracy of PnP-ADMM using mismatched denoisers by controlling the error term $\varepsilonbar_t$. This error term can be controlled by using domain adaptation techniques for decreasing the distance between mismatched and target denoisers, thus closing the gap in the performances of PnP-ADMM. We numerically validate this observation in Section~\ref{sec:numVal} by considering the fine-tuning of mismatched denoisers to the target distribution with a limited number of samples.

\section{Numerical Validation}
\label{sec:numVal}

We consider PnP-ADMM with mismatched and adapted denoisers for the task of image super-resolution. Our first set of results shows how distribution shifts relate to the prior disparities and their impact on PnP recovery performance. Our second set of results shows the impact of domain adaptation on the denoiser gap and PnP performance. We use the traditional $l_2$-norm as the data-fidelity term. To provide an objective evaluation of the final image quality, we use two established quantitative metrics: Peak Signal-to-Noise Ratio (PSNR) and Structural Similarity Index (SSIM). 

\medskip\noindent
We use DRUNet architecture~\cite{Zhang.etal2021dpir} for all image denoisers. To model prior mismatch, we train denoisers on five image datasets: MetFaces~\cite{karras2020training}, AFHQ~\cite{choi2020stargan}, CelebA~\cite{liu2015faceattributes}, BreCaHAD~\cite{aksac2019brecahad}, and RxRx1~\cite{sypetkowski2023rxrx1}. Figure~\ref{fig:gts} illustrates samples from the datasets. Our training dataset consists of 1000 randomly chosen, resized, or cropped image slices, each measuring $256 \times 256$ pixels. Unlike several existing PnP methods~\cite{Sun.etal2021, Liu.etal2021b} that suggest the inclusion of the spectral normalization layers into the CNN to enforce Lipschitz continuity on the denoisers, we directly train denoisers without any nonexpansiveness constraints. 

\subsection{Impact of Prior Mismatch}
The observation model for single image super-resolution is $\ybm = \Sbm \Hbm \xbm +\ebm$, where $\Sbm \in \R^{m\times n}$ is a standard $s$-fold downsampling matrix with $n = m \times s^2$, $\Hbm \in \R^{n \times n}$ is a convolution with anti-aliasing kernel, and $\ebm$ is the noise.  To compute the proximal map efficiently for the $l_2$-norm data-fidelity term (Step 3 in Algorithm~\ref{Alg:InexactPnPADMM}), we use the closed-form solution outlined in~\cite{Zhang.etal2021dpir, Zhao2016FastSI}. Similarly to~\cite{Zhang.etal2021dpir}, we use four isotropic kernels with different standard deviation $\{0.7, 1.2,1.6,2\}$, as well as four anisotropic kernels depicted in Table~\ref{tab:SR_results_mismatched}. We perform downsampling at scales of $s=2$ and $s = 4$.

%%%%%%
\begin{figure}[!t]
    \centering
    \includegraphics[width=0.99\textwidth]{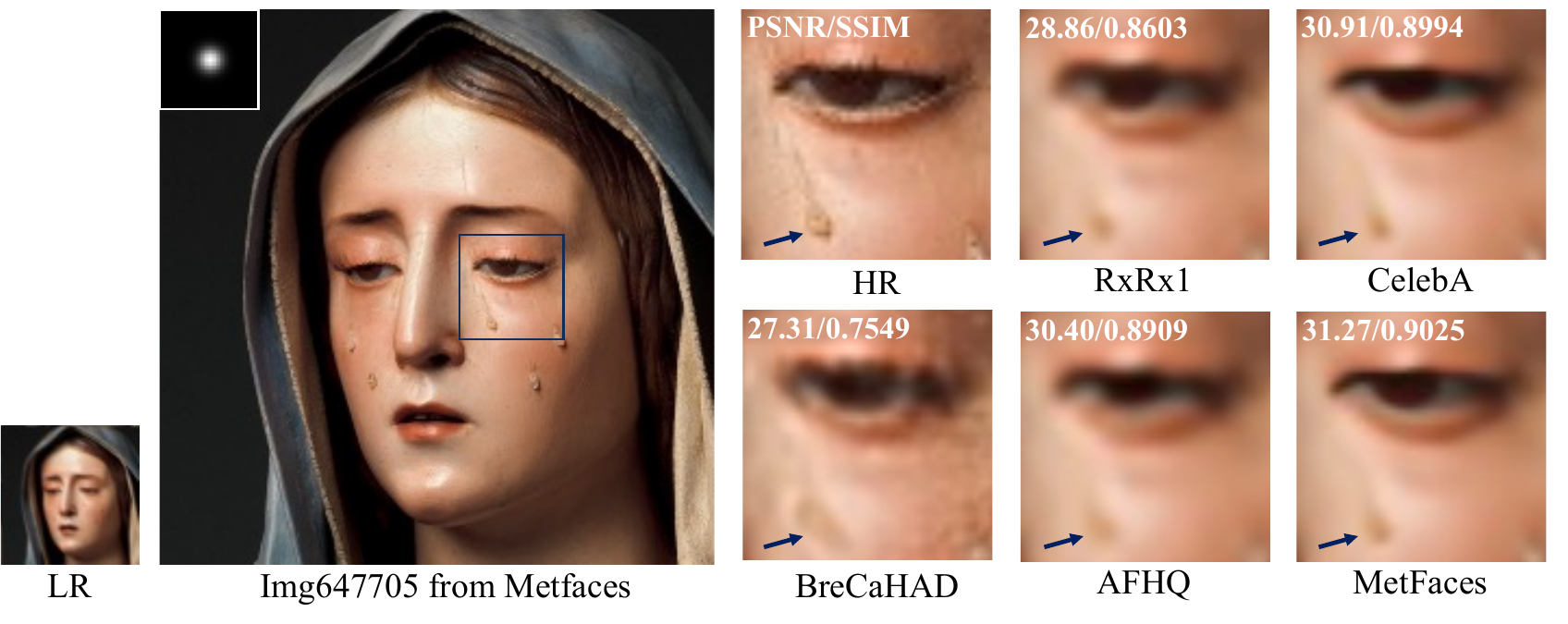}
    \caption{\emph{Visual evaluation of PnP-ADMM on image super-resolution using denoisers trained on several datasets. Performance is  reported in terms of PSNR (dB) and SSIM for an image from the MetFaces dataset. Images are downsampled by a scale of $s=4$ and convolved with the blur kernel shown on the top left corner of the ground truth image. Note how the disparities in the training distributions of denoisers directly influence the performance of PnP. The denoisers containing images most similar to MetFaces offer the best performance.}}
    \label{fig:SR_mismatched}
\end{figure}
%%%%%%%%%%%%%%%%%%
%%%%%%%%%%%%%%%%%
\begin{table}[!t]\footnotesize
 \caption{\emph{PSNR (dB) and SSIM values for image super-resolution using PnP-ADMM under different priors on a test set from the MetFaces~\cite{karras2020training}. We highlighted the \textbf{best} performing  and the {\color{lightred} \textbf{worst}} performing priors. BreCaHAD is the worst prior that is also the one visually most different from MetFaces.}}
\centering\setlength\tabcolsep{2pt}\renewcommand{\arraystretch}{.92}
    \begin{tabular}{c l c c c c c c c c c c c }
         \multirow{2}{*}{Kernels}  & \multirow{2}{*}{Prior} \ \ & \multicolumn{2}{c}{$s = 2$} & &
         \multicolumn{2}{c}{$s = 4$} & &\multicolumn{2}{c}{\textit{Avg}} \\
        \cmidrule(lr){3-5} \cmidrule(lr){6-7}  \cmidrule(lr){8-10}
         &   & \ PSNR &   SSIM &  & \ PSNR &   SSIM &  & \ PSNR &   SSIM &  \\
        \midrule
        \multirow{5}{*}{ \includegraphics[width=0.055\textwidth]{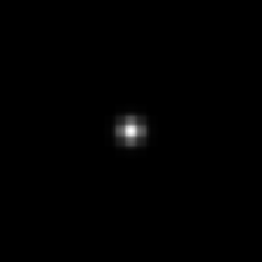}
        \includegraphics[width=0.055\textwidth]{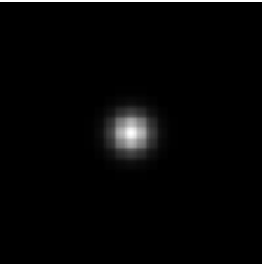}
        \includegraphics[width=0.055\textwidth]{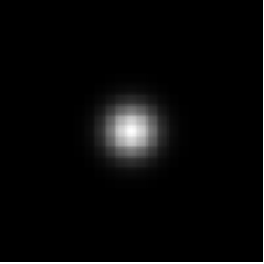}
        \includegraphics[width=0.055\textwidth]{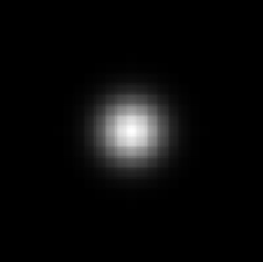}}
        & BreCaHAD  & \color{lightred}$\mathbf{31.96}$ & \color{lightred} $\mathbf{0.8108}$ && \color{lightred}$\mathbf{28.41}$ & \color{lightred}$\mathbf{0.6937}$ && \color{lightred}$\mathbf{30.18}$ & \color{lightred}$\mathbf{0.7522}$\\
        & RxRx1     & $33.45$ & $0.8683$ && $30.45$ & $0.7906$ &&  $31.95$ & $0.8294$ \\
         & AFHQ      & $33.74$ & $0.8697$ && $30.38$ & $0.7825$ && $32.06$ &$0.8261$ \\
        & CelebA    & $33.96$ & $0.8731$ && $30.62$ & $0.7906$ && $32.29$  & $0.8318$ \\
        &  MetFaces & $\mathbf{34.07}$ & $\mathbf{0.8755}$ && $\mathbf{31.15}$ & $\mathbf{0.8053}$ && $\mathbf{32.61}$  & $\mathbf{0.8404}$\\
        \midrule
        \multirow{5}{*}{ \includegraphics[width=0.055\textwidth]{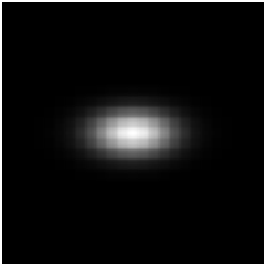}
        \includegraphics[width=0.055\textwidth]{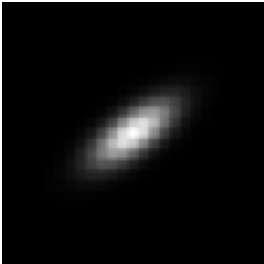}
        \includegraphics[width=0.055\textwidth]{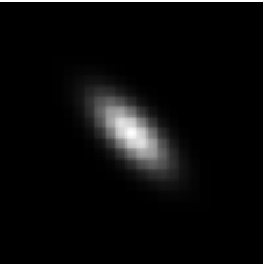}
        \includegraphics[width=0.055\textwidth]{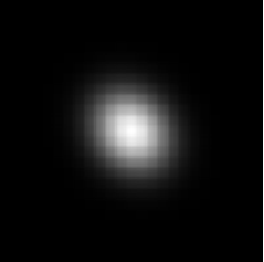}}
        & BreCaHAD  &   \color{lightred}$\mathbf{30.25}$ & \color{lightred}$\mathbf{0.7489}$ && \color{lightred}$\mathbf{28.99}$ & \color{lightred}$\mathbf{0.7083}$ && \color{lightred}$\mathbf{29.62}$ & \color{lightred}$\mathbf{0.7286}$\\
        & RxRx1     &   $32.22$ & $0.8348$ && $30.80$ & $0.7948$ && $31.51$ &$0.8148$\\
         & AFHQ      &   $32.63$ & $0.8410$ && $31.06$ & $0.8014$ && $ 31.84$ &$0.8212 $\\
        & CelebA    &   $32.62$ & $0.8404$ && $31.30$ & $0.8070$ && $31.96$  & $0.8237$ \\
        & MetFaces  &   $\mathbf{32.85}$ & $\mathbf{0.8457}$ && $\mathbf{31.44}$ & $\mathbf{0.8089}$ && $\mathbf{32.14}$ &$\mathbf{0.8273}$\\
        
        \midrule
    \end{tabular}
    \label{tab:SR_results_mismatched}
\end{table}
%%%%%%%%%%%%%%%%%%

\medskip\noindent
Figure~\ref{fig:SR_mismatched} illustrates the performance of PnP-ADMM using the target and four mismatched denoisers.
Note the suboptimal performance of PnP-ADMM using mismatched denoisers trained on the BreCaHAD, RxRx1, CelebA, and AFHQ datasets relative to PnP-ADMM using the target denoiser trained on the MetFaces dataset. 
Figure~\ref{fig:SR_mismatched} illustrates how distribution shifts can lead to mismatched denoisers, subsequently impacting the performance of PnP-ADMM. It's worth noting that the denoiser trained on the CelebA dataset~\cite{liu2015faceattributes}, which consists of facial images similar to MetFaces, is the best-performing mismatched denoiser. Table~\ref{tab:SR_results_mismatched} provides a quantitative evaluation of the PnP-ADMM performance with the target denoiser consistently outperforming all the other denoisers. Notably, the mismatched denoiser trained on the BreCaHAD dataset~\cite{aksac2019brecahad}, containing cell images that are most dissimilar to MetFaces, exhibits the worst performance.
 
\subsection{Domain Adaption}

In domain adaptation, the pre-trained mismatched denoisers are updated using a limited number of data from the target distribution. We investigate two adaptation scenarios: in the first, we adapt the denoiser initially pre-trained on the BreCaHAD dataset to the MetFaces dataset, and in the second, we use the denoiser initially pre-trained on CelebA for adaptation to the RxRx1 dataset.

\medskip\noindent
Figure~\ref{fig:disvsgap} illustrates the influence of domain adaptation on denoising and PnP-ADMM. The reported results are tested on RxRx1 and MetFaces datasets for the super-resolution task. The kernel used is shown on the top left corner of the ground truth image in Figure~\ref{fig:SR_mismatched} and the images are downsampled at the scale of $s=4$. Note how the denoising performance improves as we increase the number of images used for domain adaptation. This indicates that domain adaptation reduces the distance of mismatched and target denoisers. Additionally,  
note the direct correlation between the denoising capabilities of priors and the performance of PnP-ADMM. Figure~\ref{fig:disvsgap} shows that the performance of PnP-ADMM with mismatched denoisers can be significantly improved by adapting the mismatched denoiser to the target distribution, even with just four images from the target distribution. 

\medskip\noindent
Figure~\ref{fig:met_update} presents visual examples illustrating domain adaptation in PnP-ADMM for image super-resolution. 
The recovery performance is shown for two test images from the MetFaces using adapted denoisers against both target and mismatched denoisers. The experiment was conducted under the same settings as those in Figure~\ref{fig:SR_mismatched}. Note the effectiveness of domain adaptation in mitigating the impact of distribution shifts on PnP-ADMM. 
Table~\ref{tab:SR_results_updated} provides quantitative results of several adapted priors on the test data. The results presented in Table~\ref{tab:SR_results_updated} show the substantial impact of domain adaptation, using a limited number of data, in significantly narrowing the performance gap that emerges as a consequence of distribution shifts.

%%%%%%%%%%%%%%%%%%
\begin{figure}[!t]
    \centering
     \includegraphics[width=0.95\textwidth,center]{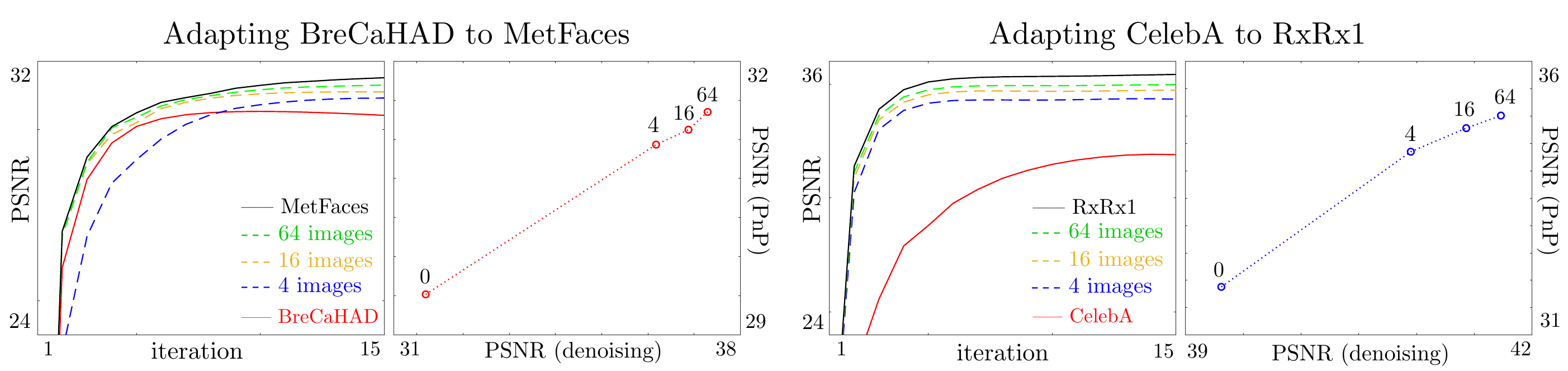}
    \caption{\emph{Illustration of prior mismatch and adaption in PnP-ADMM, where a denoiser trained on one dataset (BreCaHAD~
    \cite{aksac2019brecahad} or CelebA~\cite{liu2015faceattributes}) is used to recover an image from another dataset (MetFaces~\cite{karras2020training} or RxRx1~\cite{sypetkowski2023rxrx1}). We plot the convergence of PnP-ADMM in terms of PSNR (first and third figures) and the influence of adapted denoisers on the performance of PnP-ADMM (second and fourth figures). Note how adaptation with even few samples is enough to nearly close the performance gap in PnP-ADMM.}
    }
    \label{fig:disvsgap}
\end{figure}
%%%%%%%%%%%%%%%%%% 
%%%%%%%%%%%%%%%%%%
\begin{figure}[!t]
    \includegraphics[width=1.043\textwidth,center]{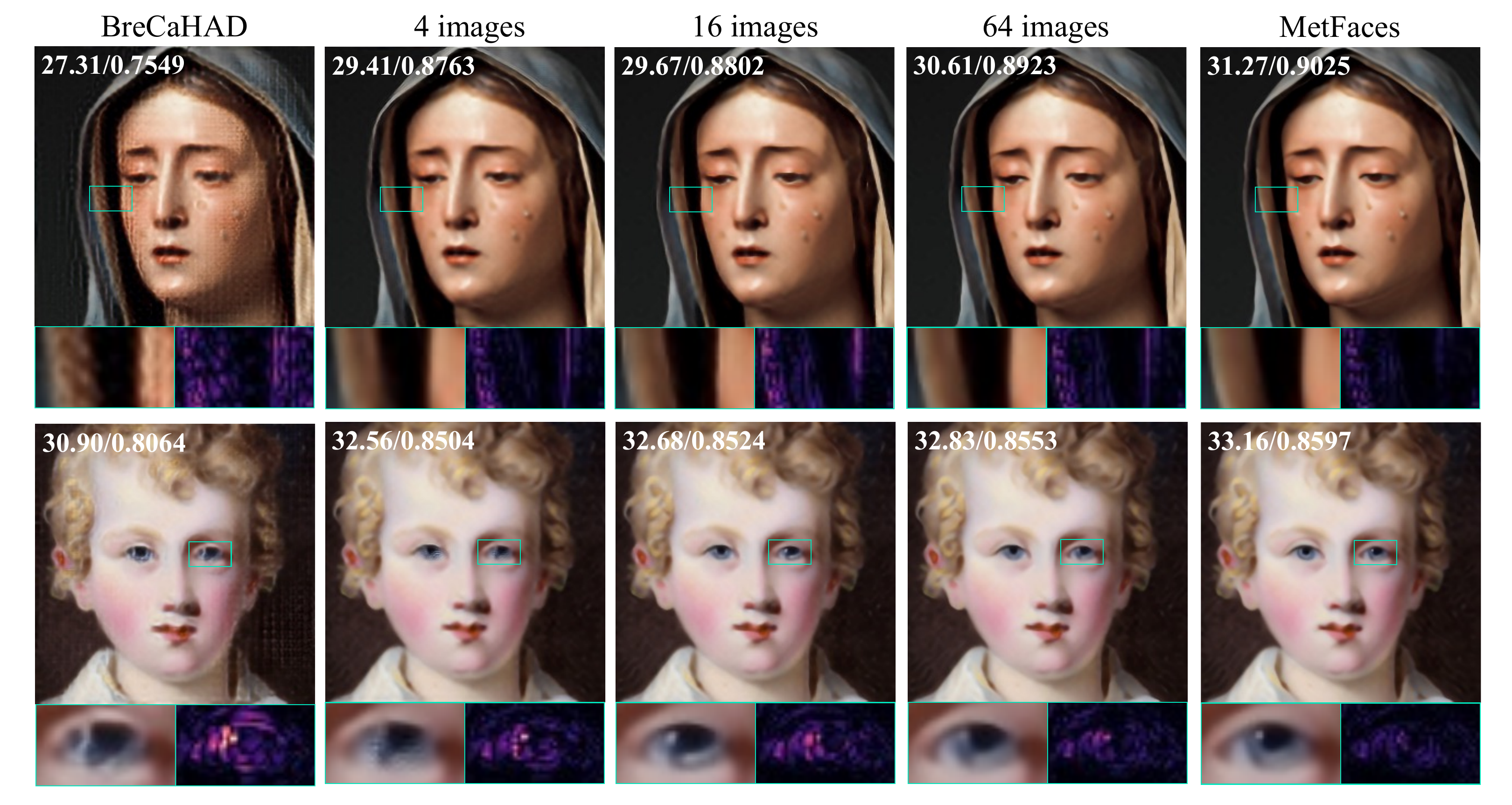}
    \caption{\emph{Visual comparison on super-resolution with target (MetFaces), mismatched (BreCaHAD), and adapted priors on two MetFaces test images. The images are downsampled by the scale of $s=4$. The performance is reported in terms of PSNR (dB) and SSIM. Note how the recovery performance increases by adaptation of mismatched priors to a larger set of images from the target distribution.}}
    \label{fig:met_update}
\end{figure}
%%%%%%%%%%%%%%%%%%
%%%%%%%%%%%%%%%%%
\begin{table}[!t]\footnotesize
 \caption{\emph{PSNR (dB) and SSIM comparison of super-resolution with mismatched, target, and adapted denoisers for the test set from MetFaces, averaged for indicated kernels. We highlighted the \textbf{target},  {\color{lightred} \textbf{mismatched}}, and the {\color{lightblue} \textbf{best}} adapted priors.
    }}
\centering\setlength\tabcolsep{2pt}\renewcommand{\arraystretch}{.92}
    \begin{tabular}{c l c c c c c c c c c c c }
         \multirow{2}{*}{Kernels}  & \multirow{2}{*}{Prior} \ \ & \multicolumn{2}{c}{$s = 2$} & &
         \multicolumn{2}{c}{$s = 4$} & &\multicolumn{2}{c}{\textit{Avg}} \\
        \cmidrule(lr){3-5} \cmidrule(lr){6-7}  \cmidrule(lr){8-10}
         &   & \ PSNR &   SSIM &  & \ PSNR &   SSIM &  & \ PSNR &   SSIM &  \\
        \midrule
        \multirow{5}{*}{ \includegraphics[width=0.055\textwidth]{Fig/kernel1.png}
        \includegraphics[width=0.055\textwidth]{Fig/kernel2.png}
        \includegraphics[width=0.055\textwidth]{Fig/kernel3.png}
        \includegraphics[width=0.055\textwidth]{Fig/kernel4.png}}
       & BreCaHAD  & \color{lightred}$\mathbf{31.96}$ & \color{lightred} $\mathbf{0.8108}$ && \color{lightred}$\mathbf{28.41}$ & \color{lightred}$\mathbf{0.6937}$ && \color{lightred}$\mathbf{30.18}$ & \color{lightred}$\mathbf{0.7522}$\\
        &  4 imgs    & $32.51$ & $0.8510$ && $30.57$ & $0.7934$ &&  $31.54$ & $0.8222$ \\
        &  16 imgs   & $33.10$ & $0.8611$ && $30.65$ & $0.7961$ &&  $31.89$ & $0.8293$ \\
        &   32 imgs  & $33.30$ & $0.8649$ && $30.81$ & $0.8001$ && $32.05$  & $0.8325$ \\
        &   64 imgs  & \color{lightblue} $\mathbf{33.59}$ & \color{lightblue}$\mathbf{0.8698}$ && \color{lightblue}$\mathbf{30.84}$ & \color{lightblue}$\mathbf{0.7994}$ &&  \color{lightblue}$\mathbf{32.21}$ & \color{lightblue}$\mathbf{0.8346}$ \\
         &  MetFaces & $\mathbf{34.07}$ & $\mathbf{0.8755}$ && $\mathbf{31.15}$ & $\mathbf{0.8053}$ && $\mathbf{32.61}$  & $\mathbf{0.8404}$\\
        \midrule
        \multirow{5}{*}{ \includegraphics[width=0.055\textwidth]{Fig/kernel5.png}
        \includegraphics[width=0.055\textwidth]{Fig/kernel6.png}
        \includegraphics[width=0.055\textwidth]{Fig/kernel7.png}
        \includegraphics[width=0.055\textwidth]{Fig/kernel8.png}}
        & BreCaHAD  &   \color{lightred}$\mathbf{30.25}$ & \color{lightred}$\mathbf{0.7489}$ && \color{lightred}$\mathbf{28.99}$ & \color{lightred}$\mathbf{0.7083}$ && \color{lightred}$\mathbf{29.62}$ & \color{lightred}$\mathbf{0.7286}$\\
        &  4 imgs    & $31.59$ & $0.8215$ && $30.86$ & $0.7957$ && $31.22$  & $0.8086$  \\
        &  16 imgs   & $32.19$ & $0.8334$ && $31.05$ & $0.8009$ && $31.62$  & $0.8171$ \\
        &  32 imgs   & $32.34$ & $0.8371$ && $31.18$ & $0.8044$ &&  $31.76$ & $0.8207$ \\
        &  64 imgs   & \color{lightblue}$\mathbf{32.47}$ &\color{lightblue} $\mathbf{0.8397}$ && \color{lightblue}$\mathbf{31.26}$ & \color{lightblue}$\mathbf{0.8059}$ &&  \color{lightblue}$\mathbf{31.86}$ &\color{lightblue} $\mathbf{0.8228}$ \\
       & MetFaces  &   $\mathbf{32.85}$ & $\mathbf{0.8457}$ && $\mathbf{31.44}$ & $\mathbf{0.8089}$ && $\mathbf{32.14}$ &$\mathbf{0.8273}$\\
       \midrule
    \end{tabular}
    \label{tab:SR_results_updated}
\end{table}
%%%%%%%%%%%%%%%%%

\section{Conclusion}
\label{sec:conclusion}
The work presented in this paper investigates the influence of using mismatched denoisers in PnP-ADMM, presents the corresponding theoretical analysis in terms of convergence, investigates the effect of mismatch on image super-resolution, and shows the ability of domain adaptation to reduce the effect of distribution mismatch. The theoretical results in this paper extend the recent PnP work by accommodating mismatched priors while eliminating the need for convex data-fidelity and nonexpansive denoiser assumptions. The empirical validation of PnP-ADMM involving mismatched priors and the domain adaptation strategy highlights the direct relationship between the gap in priors and the subsequent performance gap in the PnP-ADMM recovery, effectively reflecting the influence of distribution shifts on image priors.

\section*{Limitations}
The basis of our analysis in this study relies on the assumption that the denoiser used in the inference accurately computes the MMSE denoiser for both target and mismatched distributions. While this assumption aligns well with deep denoisers trained via the MSE loss, it does not directly extend to denoisers trained using alternative loss functions, such as the $l_1$-norm or SSIM. As is customary in theoretical research, our analysis remains valid only under the fulfillment of these assumptions, which could potentially restrict its practical applicability. In our future work, we aim to enhance the results presented here by exploring new PnP strategies that can relax these convergence assumptions.

\section*{Reproducibility Statement}
We have provided the anonymous source code in the supplementary materials. The included README.md file contains detailed instructions on how to run the code and reproduce the results reported in the paper. The algorithm's pseudo-code is outlined in Algorithm~\ref{Alg:InexactPnPADMM}, and for more comprehensive information about training, parameter selection, and other details, please refer to the Supplementary Section~\ref{sup:sec:additionalresults}. Additionally, for the theoretical findings presented in section~\ref{sec:theoanalysis}, complete proofs, along with further clarifications on the assumptions and additional contextual information, can be found in the appendices~\ref{Sup:Sec:theorem1}-\ref{sup:sec:addtionaltheory}.

\section*{Ethics Statement}
To the best of our knowledge, this work does not give rise to any significant ethical concerns.

%%% broader impact
\section*{Acknowledgement}

This paper is partially based upon work supported by the NSF CAREER awards under grants CCF-2043134 and CCF-2046293.

{\small
\bibliographystyle{IEEEbib}
\bibliography{references}
}

\newpage
\appendix

{\Large Supplementary Material}

\section{Proof of Theorem 1}\label{Sup:Sec:theorem1}

\begin{suptheorem}
Run PnP-ADMM with \textbf{mismatched} MMSE denoiser for $t \geq 1$ iterations under Assumptions~\ref{As:NonDegenerate}-\ref{As:BoundIters} with the penalty parameter  $0<\gamma\leq 1/(4L)$. Then, we have 
    \begin{equation*}
              \min_{1\leq k\leq t} \left \|\nabla f\left (\xbm^k\right )\right \|_2^2 \leq  \frac{1}{t} \sum_{k=1}^t\left\|\nabla f\left (\xbm^k\right )\right\|^2_2 
    \leq \frac{A_1}{t} \left(\phi\left (\xbm^0, \zbm^0, \sbm^0\right)-\phi^\ast\right) +  A_2 \varepsilonbar_t
    \end{equation*}
where $A_1>0$ and $A_2 >0$ are iteration independent constants, $\phi^* = \phi(\xbm^*, \zbm^*, \sbm^*)$, $\varepsilon_k \defn \max\{\delta_k, \delta_k^2\}$, and  $\varepsilonbar_t \defn (1/t)\left(\varepsilon_1 + \cdots + \varepsilon_t\right)$. In addition, if the sequence of distances of denoisers $\{\delta_i\}_{i\geq 1}$ is summable, we have that  $\|\nabla f(\xbm^t)\|_2 \to 0$  as $t \to \infty$..

\begin{proof}
    Note that from Lemma \ref{Lem:InExactLagDecreasing}, we have 
    \begin{align}\label{Eq:iterateRel}
        \nonumber \left \|\zbm^k -\zbm^{k-1} \right\|_2^2  &\leq \frac{2\gamma}{1- 2\gamma L- 2\gamma^2 L^2}\left(\phi\left (\xbm^{k-1}, \zbm^{k-1} , \sbm^{k-1}\right ) - \phi\left (\xbm^k , \zbm^k , \sbm^k\right )\right)\\
       & +  \frac{3}{4\left (1- 2\gamma L- 2 \gamma^2 L^2\right )}\ \delta_k^2+ \frac{4R}{1- 2\gamma L- 2 \gamma^2 L^2}\ \delta_k .
    \end{align}
    % By averaging over $t \geq 1$ iterations, we obtain
    % \begin{align} \label{Eq:iterateDec}
    % \nonumber\frac{1}{t}\sum_{k=1}^t \left\|\zbm^k - \zbm^{k-1}\right\|^2_2 &\leq \frac{D_1}{t} \left (\phi\left (\xbm^0, \zbm^0, \sbm^0\right )-\phi\left (\xbm^t, \zbm^t, \sbm^t\right )\right ) + D_2  \varepsilonbar^2_t  + D_3\varepsilonbar_t \\
    % & \leq \frac{D_1}{t} \left (\phi\left (\xbm^0, \zbm^0, \sbm^0 \right )-\phi^\ast\right) + D_2  \varepsilonbar^2_t  + D_3\varepsilonbar_t, 
    % \end{align}
    % where $D_1 \defn 2\gamma / (1- 2\gamma L- 2 \gamma^2 L^2)$, $D_2 \defn 3 /(4(1- 2\gamma L- 2 \gamma^2 L^2))$, and $D_3 \defn  4R/(1- 2\gamma L- 2 \gamma^2 L^2)$. Note that in the second inequality, 
      % we used the fact that augmented Lagrangian $\phi$ is bounded from below from Lemma~\ref{Lem:BoundBelowInExact}. Moreover, $\varepsilonbar_t^2 \defn (1/t) \left(\varepsilon_1^2 + \cdots + \varepsilon_t^2\right)$ and $\varepsilonbar_t\defn (1/t) \left(\varepsilon_1 + \cdots + \varepsilon_t\right)$ correspond to the mean squared-distance and distance between the mismatched and matched denoisers.
From the optimality conditions of the target MMSE denoiser $\Dsf_\sigma$ and  proximal operator,  for $\zbmbar^k \in \Im(\Dsf_\sigma)$, we have
    \begin{equation*}
        \nabla h  \left (\zbmbar^k\right )  + \frac{1}{\gamma} \left ( \zbmbar^k -\xbm^k -\sbm^{k-1}\right ) = \bm{0} \quad \text{and} \quad  \nabla g\left (\xbm^k\right ) + \frac{1}{\gamma} \left  ( \xbm^k + \sbm^{k-1} - \zbm^{k-1} \right ) = \bm{0}.
    \end{equation*}
From this equation, for the objective function defined in~\eqref{Eq:OptimInverseProb}, we can write
\begin{align}
   \nonumber \left\|\nabla f\left (\xbm^k\right ) \right\|_2 
   \nonumber&=\left \|\nabla g\left (\xbm^k\right ) + \nabla h\left (\xbm^k\right )\right\|_2 \\
   \nonumber&= \left\|\nabla g\left (\xbm^k\right ) + \frac{1}{\gamma} \left (\xbm^k +\sbm^{k-1} - \zbm^{k-1}\right )+ \nabla h\left (\xbm^k\right ) + \frac{1}{\gamma} \left (\zbm^{k-1}- \xbm^k -\sbm^{k-1} \right )\right\|_2 \\
   \nonumber&= \left\| \nabla h\left (\zbmbar^k\right ) + \frac{1}{\gamma} \left (\zbmbar^k- \xbm^k -\sbm^{k-1} \right ) + \nabla h\left (\xbm^k\right ) - \nabla h\left (\zbmbar^k \right ) + \frac{1}{\gamma} \left (\zbm^{k-1} -\zbmbar^k \right )\right\|_2 \\
   \nonumber&= \left\| \nabla h\left (\xbm^k\right ) - \nabla h\left (\zbm^k \right ) + \nabla h\left (\zbm^k\right ) - \nabla h\left (\zbmbar^k \right )+ \frac{1}{\gamma} \left (\zbm^{k-1} -\zbmbar^k \right )\right\|_2 \\
  \nonumber&\leq \left\| \nabla h\left (\xbm^k\right ) - \nabla h\left (\zbm^k \right ) \right\|_2+\left\| \nabla h\left (\zbm^k\right ) - \nabla h\left (\zbmbar^k \right )\right\|_2 + \frac{1}{\gamma} \left\|\zbm^{k-1} -\zbmbar^k \right\|_2 \\
  \nonumber&\leq  L \left\| \xbm^k - \zbm^k \right\|_2+ L\left \| \zbm^k - \zbmbar^k \right\|_2 + \frac{1}{\gamma}\left \|\zbm^{k-1} -\zbm^k \right\|_2  +\frac{1}{\gamma} \left\|\zbm^{k} -\zbmbar^k \right\|_2  \\
  \nonumber &\leq \left  ( \frac{1}{\gamma } +\gamma  L^2 \right  ) \left \| \zbm^k - \zbm^{k-1}\right\|_2 + \left (  \frac{1}{\gamma } + L \right ) \delta_k, 
\end{align}
    where we used  triangle inequality in the first and second inequality. We also used 
    $$
    \left\| \xbm^k - \zbm^k \right\|_2 \leq \left\| \sbm^k - \sbm^{k-1} \right\|_2 = \gamma \left\| \nabla h(\zbm^k) - \nabla h (\zbm^{k-1}) \right\|_2 \leq \gamma L\left\| \zbm^k - \zbm^{k-1} \right\|_2
    $$
    and Assumption~\ref{As:InexactDistance} in the last inequality.  
    By squaring both sides and using $(a+b)^2 \leq 2a^2 + 2b^2$, we get 
    \begin{equation*}
        \nonumber \left \|\nabla f\left (\xbm^k\right )\right\|^2_2 
   \leq 2\left ( \frac{1}{\gamma } + \gamma L^2 \right )^2 \left\|\zbm^k - \zbm^{k-1}\right\|^2_2 + 2\left(  \frac{1}{\gamma } + L \right )^2 \delta_k^2.
    \end{equation*}
    By using the result from equation~\eqref{Eq:iterateRel}, we obtain
    \begin{align}
        \label{Eq:finalGradBound}
\nonumber \left \|\nabla f\left (\xbm^k\right )\right\|^2_2 
&\leq  \frac{4\left (  1+ \gamma^2 L^2\right )^2}{\gamma \left(1- 2\gamma L- 2\gamma^2 L^2\right)} \left(\phi\left (\xbm^{k-1}, \zbm^{k-1} , \sbm^{k-1}\right ) - \phi\left (\xbm^k , \zbm^k , \sbm^k\right )\right)  \\ 
&  + \left (\frac{3\left (  1+ \gamma^2 L^2\right )^2}{2\gamma^2 \left(1- 2\gamma L- 2\gamma^2 L^2\right)} + \frac{2\left(1+\gamma L\right)^2}{\gamma^2}  \right )\delta_k^2 + \frac{8 R\left (  1+ \gamma^2 L^2\right )^2}{\gamma^2 \left(1- 2\gamma L- 2\gamma^2 L^2\right)} \delta_k.
    \end{align}
    By averaging both sides of the bound over $t\geq 1$ and using the definition of error in $\varepsilon_k \defn \max\{\delta_k, \delta_k^2\}$, we get 
    \begin{align}\label{Eq:finalBoundonGrad}
      \nonumber \min_{1\leq k\leq t} \left \|\nabla f\left (\xbm^k\right )\right\|_2^2 \leq  \frac{1}{t} \sum_{k=1}^t \left \|\nabla f\left (\xbm^k\right ) \right\|^2_2 
   \nonumber &\leq \frac{A_1}{t} \left (\phi\left (\xbm^0, \zbm^0, \sbm^0\right )-\phi\left (\xbm^t, \zbm^t, \sbm^t\right )\right) +  A_2 \varepsilonbar_t\\
    &\leq \frac{A_1}{t} \left (\phi\left (\xbm^0, \zbm^0, \sbm^0\right )-\phi^\ast\right) +  A_2 \varepsilonbar_t, 
    \end{align}
    where $\varepsilonbar_t\defn (1/t) \left(\varepsilon_1 + \cdots + \varepsilon_t\right)$, $A_1 \defn 4\left (  1+ \gamma^2 L^2\right )^2 / \left(\gamma \left(1- 2\gamma L- 2\gamma^2 L^2\right)\right)$, 
    
    \noindent$A_2 \defn \left ( \left(3+16R\right)\left(1+\gamma^2L^2\right)/\left(2\gamma^2 \left(1- 2\gamma L- 2\gamma^2 L^2\right) \right) + 2\left(1/\gamma + L\right)^2\right) $, and we used the fact that augmented Lagrangian $\phi$ is bounded from below $\left(  \phi^* \leq \phi\left(\xbm^t, \zbm^t, \sbm^t\right) \right )$ from Lemma~\ref{Lem:BoundBelowInExact} .
    Note that we used the following inequality to get the result in equation~\eqref{Eq:finalBoundonGrad}
    \begin{align*}
       &\left (\frac{3\left (  1+ \gamma^2 L^2\right )^2}{2\gamma^2 \left(1- 2\gamma L- 2\gamma^2 L^2\right)} + 2\left(  \frac{1}{\gamma } + L \right )^2  \right )\delta_k^2 + \frac{8 R\left (  1+ \gamma^2 L^2\right )^2}{\gamma^2 \left(1- 2\gamma L- 2\gamma^2 L^2\right)} \delta_k \\
       & \leq \max \{\delta_k, \delta_k^2\}\left (\frac{3\left (  1+ \gamma^2 L^2\right )^2}{2\gamma^2 \left(1- 2\gamma L- 2\gamma^2 L^2\right)} + 2\left(  \frac{1}{\gamma } + L \right )^2  + \frac{8 R\left (  1+ \gamma^2 L^2\right )^2}{\gamma^2 \left(1- 2\gamma L- 2\gamma^2 L^2\right)} \right )\\
       & = \left (\frac{\left (3+16R\right)\left (  1+ \gamma^2 L^2\right )^2}{2\gamma^2 \left(1- 2\gamma L- 2\gamma^2 L^2\right)} + 2\left(  \frac{1}{\gamma } + L \right )^2   \right ) \varepsilon_k .
    \end{align*}
Note that if the sequence of distances of denoisers $\{\delta_i \}_{i\geq 1}$ is summable, then $\{\varepsilon_i = \max \{\delta_i, \delta_i^2\}\}_{i\geq 1}$ is also be summable. Consequently, $\|\nabla f(\xbm^t)\|_2 \to 0$  as $t \to \infty$. 
\end{proof}
\end{suptheorem}

\medskip\noindent
\textbf{Remark 1.}  
Note that by using~\eqref{Eq:iterateRel} when the sequence $\{\delta_i \}_{i\geq 1}$ is summable, we have  
    \begin{equation}
        \frac{1}{t}\sum_{k=1}^t \left\|\zbm^k - \zbm^{k-1} \right\|^2_2 \leq 0 \quad \text{as} \quad t\to \infty, 
    \end{equation}
    which ensures that $\|\zbm^k - \zbm^{k-1}\|_2 \to 0 $ as $k \to \infty$. Since 
    \begin{equation} \label{Eq:boundtozero}
        \left \|\sbm^k - \sbm^{k-1} \right\|_2 \leq \gamma L \left\|\zbm^k - \zbm^{k-1} \right\|_2 \quad \text{and} \quad  \left\|\sbm^k - \sbm^{k-1} \right\|_2 = \left\|\xbm^{k} - \zbm^{k}\right\|_2, 
    \end{equation}
    we conclude that $\|\xbm^k - \xbm^{k-1}\|_2 \to 0 $ and $\|\sbm^k - \sbm^{k-1}\|_2 \to 0$ as $k \to \infty$. 

%%%%%%%%%%%%%%%%%%%%%%%%%%%%%%%%%%%%%%%%%%%%%%%%%%%%%%%%%%%%%
\section{Useful results for Theorem 1}\label{Sup:Sec:lemmas}

\begin{lemma}\label{Lem:InExactLagDecreasing}
Assume that Assumptions~\ref{As:NonDegenerate}-\ref{As:BoundIters} hold and let the sequence $\{\xbm^k , \zbm^k , \sbm^k\}$ be generated via iterations of PnP-ADMM with \textbf{mismatched} MMSE denoiser using the penalty parameter  $0<\gamma\leq 1/(4L)$. Then for the augmented Lagrangian defined in~\eqref{Eq:AugLagrangian},  we have that
    \begin{equation*}
    \phi\left(\xbm^k , \zbm^k , \sbm^k \right) \leq \phi\left(\xbm^{k-1}, \zbm^{k-1} , \sbm^{k-1}\right) 
    - \left (\frac{1- 2\gamma L- 2\gamma^2 L^2}{2\gamma} \right ) \left \|\zbm^k -\zbm^{k-1} \right \|_2^2 + \frac{3}{8\gamma} \delta_k^2 + \frac{2R}{\gamma} \delta_k.
\end{equation*}
where $R$ is defined in Assumption~\ref{As:BoundIters}. 
\end{lemma}

\begin{proof}

From the smoothness of  $\hine$  for any $\zbm^k\in \Im(\Dsf_\sigma)$ in Assumption~\ref{As:LipschitzPrior}, the optimality condition for the mismatched MMSE denoiser, and the Lagrange multiplier update rule in the form of $\sbm^k = \sbm^{k-1} +\xbm^k - \zbm^k$, we have
\begin{equation*}
    \nabla \hine \left (\zbm^k\right ) = \frac{1}{\gamma} \left (\sbm^{k-1} + \xbm^k - \zbm^k\right ) = \frac{1}{\gamma} \sbm^k
\end{equation*}
and 
\begin{equation}\label{Eq:boundonSinexact}
    \left\|\sbm^k - \sbm^{k-1} \right\|_2 =\left \|\gamma \nabla \hine \left (\zbm^k\right ) - \gamma \nabla \hine \left (\zbm^{k-1}\right ) \right \|_2 \leq \gamma L\left\|\zbm^k - \zbm^{k-1}\right\|_2, 
\end{equation}
where we used  $L-$Lipschitz continuity of $\nabla \hine$ from Assumption~\eqref{As:LipschitzPrior} in the last inequality.
From this equation and the Lagrange multiplier update rule, we have
\begin{align}\label{Eq:LagDecInequality1}
    \nonumber\phi \left (\xbm^k , \zbm^k , \sbm^k\right) - \phi\left(\xbm^k , \zbm^k , \sbm^{k-1}\right) 
    \nonumber&= \frac{1}{\gamma} \left ( \sbm^k - \sbm^{k-1}\right)^\Tsf \left(\xbm^k - \zbm^k\right) = \frac{1}{\gamma} \left\|\sbm^k - \sbm^{k-1} \right\|_2^2\\
    & \leq \gamma L^2 \left\|\zbm^k - \zbm^{k-1} \right\|_2^2.
\end{align}
From the fact that $h$ (regularizer associated with target MMSE denoiser $\Dsf_\sigma$) has a $L-$Lipschitz continuous gradient over the set $\Im(\Dsf_\sigma)$ (Assumption~\ref{As:LipschitzPrior}), we have 
\begin{equation}\label{Eq:LipGrad1}
     h\left (\zbmbar^k\right ) - h\left (\zbm^{k-1}\right ) \leq \nabla h\left (\zbmbar^k\right ) ^\Tsf \left (\zbmbar^k - \zbm^{k-1}\right ) + \frac{L}{2}\left\|\zbmbar^k - \zbm^{k-1}\right\|_2^2, 
\end{equation}
where $\zbmbar^k = \Dsf_\sigma (\xbm^k+\sbm^{k-1})$. From the smoothness of $h$ for any $\zbmbar^k \in \Im(\Dsf_\sigma)$, the optimality condition for mismatched MMSE denoiser, and the Lagrange multiplier update rule $\sbm^k = \sbm^{k-1} +\xbm^k - \zbm^k$, we have
\begin{equation}\label{Eq:LipGrad2}
    \nabla h\left (\zbmbar^k\right ) = \frac{1}{\gamma} \left (\xbm^k + \sbm^{k-1} - \zbmbar^k \right ) = \frac{1}{\gamma} \sbm^k + \frac{1}{\gamma}\left (\zbm^k- \zbmbar^k\right ).
\end{equation}
By combining equations~\eqref{Eq:LipGrad1} and~\eqref{Eq:LipGrad2}, we obtain
\begin{equation}
    \label{Eq:RegInequality}
    h\left (\zbmbar^k\right ) - h\left (\zbm^{k-1}\right ) \leq \frac{1}{\gamma} \left (\sbm^k\right )^\Tsf \left (\zbmbar^k - \zbm^{k-1}\right ) + \frac{1}{\gamma} \left (\zbm^k- \zbmbar^k\right )^\Tsf \left (\zbmbar^k - \zbm^{k-1}\right )  + \frac{L}{2}\left\|\zbmbar^k - \zbm^{k-1}\right\|_2^2.
\end{equation}
For the  target MMSE denoiser $\Dsf_\sigma$, we know that $\zbmbar^k \in \Im(\Dsf_\sigma)$ minimizes
\begin{equation}\label{Eq:FuncPsiDef1}
\psi_{\gamma h}\left (\zbm\right ) \defn \frac{1}{2\gamma}\left\|\zbm - \left (\xbm^k + \sbm^{k-1}\right )\right\|_2^2 + h\left (\zbm\right ).   
\end{equation}
From Assumption \ref{As:LipschitzPrior}, we know that $\nabla h$ is $L-$Lipschitz continuous over $\Im(\Dsf_\sigma)$, which implies 
\begin{equation*}
    \left\|\nabla \psi_{\gamma h} \left (\ubm\right ) - \nabla \psi_{\gamma h} \left (\vbm\right )\right\|_2 \leq \left  (\frac{1}{\gamma} + L\right ) \left\|\ubm - \vbm\right\|_2 \quad \forall \ubm, \vbm \in \Im(\Dsf_\sigma). 
\end{equation*}
From the smoothness of $\psi_{\gamma h}$ and the fact that $\zbmbar^k$ minimizes it, we have 
\begin{align*}
    \psi_{\gamma h} \left (\zbm^k\right ) &\leq \psi_{\gamma h} \left (\zbmbar^k\right ) + \nabla \psi_{\gamma h} \left (\zbmbar^k\right )^\Tsf \left (\zbm^k - \zbmbar^k\right ) + \left (\frac{1}{2\gamma} + \frac{L}{2}\right ) \left\|\zbm^k - \zbmbar^k\right\|_2^2\\
    & = \psi_{\gamma h} \left (\zbmbar^k\right ) + \left(\frac{1}{2\gamma} + \frac{L}{2}\right) \left \|\zbm^k - \zbmbar^k\right\|_2^2.
\end{align*}
By using the definition of function $\psi_{\gamma h}$ in~\eqref{Eq:FuncPsiDef1}, the Lagrange multiplier update rule $\sbm^k = \sbm^{k-1} + \xbm^k - \zbm^k$ and rearranging the terms, we obtain
\begin{align} \label{Eq:RegInequality2}
    h\left (\zbm^k\right ) - h\left (\zbmbar^k\right ) 
    \nonumber&\leq \frac{1}{2\gamma}\left \|\zbmbar^k - \left (\xbm^k + \sbm^{k-1}\right )\right\|_2^2 - \frac{1}{2\gamma}\left\|\zbm^k - \left (\xbm^k + \sbm^{k-1}\right )\right\|_2^2 + \left(\frac{1}{2\gamma} + \frac{L}{2}\right)\left\|\zbm^k - \zbmbar^k\right\|_2^2\\
    \nonumber&=  \frac{1}{2\gamma} \left (\zbmbar^k + \zbm^k - 2\left (\xbm^k + \sbm^{k-1}\right )\right )^\Tsf\left (\zbmbar^k - \zbm^k\right ) +  \left (\frac{1}{2\gamma} + \frac{L}{2}\right)\left\|\zbm^k - \zbmbar^k\right\|_2^2 \\
    \nonumber&= \frac{1}{\gamma} \left (\sbm^k\right )^\Tsf \left (\zbm^k - \zbmbar^k\right ) + \frac{1}{2\gamma} \left \|\zbm^k - \zbmbar^k\right\|_2^2 +\left(\frac{1}{2\gamma} + \frac{L}{2}\right) \left\|\zbm^k - \zbmbar^k\right\|_2^2 \\
    &= \frac{1}{\gamma} \left (\sbm^k\right )^\Tsf \left (\zbm^k - \zbmbar^k\right ) + \left(\frac{1}{\gamma} + \frac{L}{2}\right) \left \|\zbm^k - \zbmbar^k\right\|_2^2. 
\end{align}
Now for the augmented Lagrangian, we have 
\begin{align}\label{Eq:RegInequality3}
    \nonumber \phi\left (\xbm^k , \zbm^k , \sbm^{k-1}\right )  - \phi\left (\xbm^k , \zbm^{k-1} , \sbm^{k-1}\right )
     \nonumber &=  h\left (\zbm^k\right ) - h\left (\zbm^{k-1}\right ) + \frac{1}{\gamma} \left ({\sbm^{k-1}}\right )^\Tsf \left (\zbm^{k-1} - \zbm^k\right )\\
     \nonumber&\quad + \frac{1}{2\gamma}\left (2\xbm^k - \zbm^k - \zbm^{k-1}\right )^\Tsf \left (\zbm^{k-1} - \zbm^k\right )\\
    \nonumber& =  h\left (\zbm^k\right ) - h\left (\zbm^{k-1}\right )+ \frac{1}{\gamma} \left ({\sbm^{k-1}}\right )^\Tsf \left (\zbm^{k-1} - \zbm^k\right ) \\
    \nonumber&\quad + \frac{1}{\gamma}\left (\sbm^k - \sbm^{k-1}\right )^\Tsf \left (\zbm^{k-1} - \zbm^k\right ) - \frac{1}{2\gamma}\left\|\zbm^k - \zbm^{k-1}\right\|^2_2\\
    \nonumber& =  h\left (\zbm^k\right ) -  h\left (\zbmbar^k\right ) +h\left (\zbmbar^k\right ) - h\left (\zbm^{k-1}\right )\\
    &\quad + \frac{1}{\gamma} \left ({\sbm^{k}}\right )^\Tsf \left (\zbm^{k-1} - \zbm^k\right ) - \frac{1}{2\gamma}\left\|\zbm^k - \zbm^{k-1}\right\|^2_2, 
\end{align}
where we used the Lagrange multiplier update rule in the second equality. By plugging~\eqref{Eq:RegInequality} and~\eqref{Eq:RegInequality2} into~\eqref{Eq:RegInequality3} and rearranging the terms, we obtain
\begin{align}
\label{Eq:RegInequality4}
   \nonumber \phi\left (\xbm^k , \zbm^k , \sbm^{k-1}\right )& - \phi\left (\xbm^k , \zbm^{k-1} , \sbm^{k-1}\right )\leq  \frac{1}{\gamma} \left (\zbm^k- \zbmbar^k\right )^\Tsf \left (\zbmbar^k - \zbm^{k-1}\right )  + \frac{L}{2}\left\|\zbmbar^k - \zbm^{k-1}\right\|_2^2 \\
   \nonumber &\quad - \frac{1}{2\gamma}\left\|\zbm^k - \zbm^{k-1}\right\|^2_2+ \left (\frac{1}{\gamma} + \frac{L}{2}\right) \left\|\zbm^k - \zbmbar^k\right\|_2^2\\ 
   \nonumber & =  \frac{1}{\gamma} \left (\zbm^k- \zbmbar^k\right )^\Tsf \left (\zbmbar^k - \zbm^{k} + \zbm^k - \zbm^{k-1}\right )  +  \frac{L}{2}\left\|\zbmbar^k - \zbm^k + \zbm^k - \zbm^{k-1}\right\|_2^2 \\
    &\quad - \frac{1}{2\gamma}\left\|\zbm^k - \zbm^{k-1}\right\|^2_2+ \left(\frac{1}{\gamma} + \frac{L}{2}\right) \left\|\zbm^k - \zbmbar^k\right\|_2^2.
\end{align}
By using $\|\abm+\bbm\|^2 \leq 2\|\abm\|^2+ 2\|\bbm\|^2$, we can write
\begin{align}
\label{Eq:RegInequality5}
   \nonumber &\phi\left (\xbm^k , \zbm^k , \sbm^{k-1}\right ) - \phi\left (\xbm^k , \zbm^{k-1} , \sbm^{k-1}\right ) \leq  \frac{1}{\gamma} \left (\zbm^k- \zbmbar^k\right )^\Tsf \left (\zbmbar^k - \zbm^{k}\right ) +  \frac{1}{\gamma} \left (\zbm^k- \zbmbar^k\right )^\Tsf \left (\zbm^k - \zbm^{k-1}\right ) \\
  \nonumber &\quad\quad\quad~+ L \left\|\zbmbar^k - \zbm^k\right \|_2^2 + L \left\|\zbm^k - \zbm^{k-1}\right\|_2^2- \frac{1}{2\gamma}\left\|\zbm^k - \zbm^{k-1}\right\|^2_2+ \left(\frac{1}{\gamma} + \frac{L}{2}\right)\left\|\zbm^k - \zbmbar^k\right\|_2^2 \\
   \nonumber &\quad\quad\quad \leq   - \frac{1}{\gamma}\left\|\zbm^k- \zbmbar^k\right\|_2^2 + \frac{1}{\gamma} \left (\zbm^k- \zbmbar^k\right )^\Tsf \left (\zbm^k - \zbm^{k-1}\right ) - \left (\frac{1-2\gamma L}{2\gamma}\right ) \left\|\zbm^k - \zbm^{k-1}\right\|_2^2 \\
   \nonumber & \quad\quad\quad~ + \frac{1}{2\gamma}\left\|\zbm^k- \zbmbar^k\right\|_2^2 +\frac{11}{8\gamma} \left\|\zbm^k - \zbmbar^k\right\|_2^2\\
   &\quad\quad\quad= \frac{1}{\gamma} \left (\zbm^k- \zbmbar^k\right )^\Tsf \left (\zbm^k - \zbm^{k-1}\right ) - \left (\frac{1-2\gamma L}{2\gamma}\right ) \left\|\zbm^k - \zbm^{k-1}\right\|_2^2+ \frac{3}{8\gamma} \left \|\zbm^k - \zbmbar^k\right\|_2^2, 
\end{align}
where we used the fact that  $ 0< \gamma \leq 1/ (4L) $ in the second inequality.
From Assumption~\ref{As:BoundIters} and triangle inequality, we can write 
    \begin{equation}\label{Eq:2Rbound}
        \left\|\zbm^k - \zbm^{k-1}\right\|_2\leq \left\|\zbm^k - \zbm^\ast\right\|_2+\left\|\zbm^{k-1} - \zbm^\ast \right\|_2 \leq 2R,  
    \end{equation}
    where $\zbm^*$ is the stationary point of the augmented Lagrangian. 
Now by using this equation,  Assumption~\ref{As:BoundIters}, and the bound on denoisers distance in Assumption~\ref{As:InexactDistance}, we obtain
\begin{align}
\label{Eq:LagDecInequality2}
   \nonumber \phi\left (\xbm^k , \zbm^k , \sbm^{k-1}\right ) &\leq  \phi\left (\xbm^k , \zbm^{k-1} , \sbm^{k-1}\right )  +\frac{1}{\gamma} \left (\zbm^k- \zbmbar^k\right )^\Tsf \left (\zbm^k - \zbm^{k-1}\right ) \\
   \nonumber &\quad\quad - \left (\frac{1-2\gamma L}{2\gamma}\right ) \left\|\zbm^k - \zbm^{k-1}\right\|_2^2+ \frac{3}{8\gamma} \left\|\zbm^k - \zbmbar^k\right\|_2^2 \\
   \nonumber & \leq \phi\left (\xbm^k , \zbm^{k-1} , \sbm^{k-1}\right )  - \left (\frac{1-2\gamma L}{2\gamma}\right ) \left\|\zbm^k - \zbm^{k-1}\right\|_2^2 \\
   \nonumber & \quad \quad +  \frac{1}{\gamma}\left \|\zbm^k- \zbmbar^k\right\|_2 \left\|\zbm^k - \zbm^{k-1}\right\|_2 + \frac{3}{8\gamma} \left\|\zbm^k - \zbmbar^k\right\|_2^2\\
    & \leq \phi\left (\xbm^k , \zbm^{k-1} , \sbm^{k-1}\right )  - \left (\frac{1-2\gamma L}{2\gamma}\right )\left\|\zbm^k - \zbm^{k-1}\right\|_2^2 + \frac{3}{8\gamma} \delta_k^2+  \frac{2R}{\gamma} \delta_k.
\end{align}
 Note that from $\xbm^k = \prox_{\gamma g} (\zbm^{k-1} - \sbm^{k-1})$, we have 
 \begin{align*}
    \frac{1}{2\gamma}\left\|\xbm^{k} -\zbm^{k-1} + \sbm^{k-1}\right\|_2^2 + g\left (\xbm^k \right ) &=  \min_{\xbm \in \R^n} \{\frac{1}{2\gamma}\left\|\xbm -\zbm^{k-1} + \sbm^{k-1}\right\|_2^2 + g\left (\xbm\right )\}\\
    & \leq \frac{1}{2\gamma}\left\|\xbm^{k-1} -\zbm^{k-1} + \sbm^{k-1}\right\|_2^2 + g\left (\xbm^{k-1} \right ), 
 \end{align*}
which implies that 
\begin{equation}\label{Eq:LagDecInequality3}
    \phi\left (\xbm^k, \zbm^{k-1}, \sbm^{k-1}\right ) \leq \phi\left (\xbm^{k-1}, \zbm^{k-1}, \sbm^{k-1}\right ).
\end{equation}

\medskip\noindent
By combining equations~\eqref{Eq:LagDecInequality1},~\eqref{Eq:LagDecInequality2} and~\eqref{Eq:LagDecInequality3}, we obtain
    \begin{align*}
    \phi\left (\xbm^k , \zbm^k , \sbm^k\right ) \leq \phi\left (\xbm^{k-1}, \zbm^{k-1} , \sbm^{k-1}\right ) 
    - \left (\frac{1- 2\gamma L- 2\gamma^2 L^2}{2\gamma} \right ) \left\|\zbm^k -\zbm^{k-1} \right\|_2^2 + \frac{3}{8\gamma} \delta_k^2 + \frac{2R}{\gamma} \delta_k.
\end{align*}
% The final result is obtained by noting that 
% \begin{equation*}
% \frac{3}{8\gamma} \varepsilon_k^2 + \frac{2R}{\gamma} \varepsilon_k \leq \max\{\varepsilon_k, \varepsilon_k^2\} \left(\frac{3}{8\gamma}+ \frac{2R}{\gamma}\right).
% \end{equation*}
\end{proof}

\begin{lemma}\label{Lem:BoundBelowInExact}
Assume that Assumptions~\ref{As:NonDegenerate}-\ref{As:BoundIters} hold and let the sequence $\{\xbm^k, \zbm^k , \sbm^k\}$ be generated via PnP-ADMM with \textbf{mismatched} MMSE denoiser using penalty parameter $0 < \gamma \leq 1/(4L)$. Then, the augment Lagrangian $\phi$ defined in~\eqref{Eq:AugLagrangian} is bounded from below
\begin{equation*}
    \inf_{\scriptscriptstyle{k\geq 1}}\phi\left (\xbm^k, \zbm^k, \sbm^k\right ) > -\infty. 
\end{equation*}
\end{lemma}
\begin{proof}
From the smoothness of  $h$ (regularizer associated with the target denoiser $\Dsf_\sigma$) for any $\zbmbar^k \in \Im(\Dsf_\sigma)$, the optimality condition for the target MMSE denoiser, and the Lagrange multiplier update rule in the form of $\sbm^k = \sbm^{k-1} +\xbm^k - \zbm^k$, we have
    \begin{equation}
    \label{Eq:OptimProx}
    \nabla h\left (\zbmbar^k\right ) = \frac{1}{\gamma} \left (\sbm^{k-1} + \xbm^k - \zbmbar^k \right )= \frac{1}{\gamma}\sbm^{k} + \frac{1}{\gamma} \left (\zbm^k - \zbmbar^k\right ). 
    \end{equation}
    By using  the Lipschitz continuity of $\nabla h$ in Assumption~\ref{As:LipschitzPrior} and the fact that $\gamma\leq 1/(4L) < 1/L$, we have
    \begin{align*}
         h\left (\xbm^k\right ) &\leq h\left (\zbm^k\right ) + \nabla h\left (\zbm^k\right )^\Tsf \left (\xbm^k - \zbm^k\right ) + \frac{L}{2}\left \|\xbm^k - \zbm^k\right\|_2^2 \\
        &< h\left (\zbm^k\right ) + \nabla h\left (\zbm^k\right )^\Tsf \left (\xbm^k - \zbm^k\right ) + \frac{1}{2\gamma}\left\|\xbm^k - \zbm^k\right\|_2^2.
    \end{align*}
    By using this inequality and equation~\eqref{Eq:OptimProx}, we can write
       \begin{align}\label{Eq:boundlagnew}
        \nonumber\phi\left (\xbm^k, \zbm^k, \sbm^k\right ) &= g\left (\xbm^k\right ) + h\left (\zbm^k\right ) + \frac{1}{\gamma}\left (\sbm^k\right )^\Tsf \left (\xbm^k - \zbm^k\right ) + \frac{1}{2\gamma} \left\|\xbm^k - \zbm^k\right\|_2^2\\
       \nonumber & = g\left (\xbm^k\right ) + h\left (\zbm^k\right ) + \nabla h\left (\zbmbar^k\right )^\Tsf  \left (\xbm^k - \zbm^k\right ) + \frac{1}{\gamma}\left (\zbmbar^k - \zbm^k\right )^\Tsf \left (\xbm^k - \zbm^k\right ) + \frac{1}{2\gamma} \left\|\xbm^k - \zbm^k \right\|_2^2\\
         \nonumber& =  g\left (\xbm^k\right ) +h\left (\zbm^k\right ) + \nabla h\left (\zbm^k\right )^\Tsf \left (\xbm^k - \zbm^k\right ) + \frac{1}{2\gamma}\left\|\xbm^k - \zbm^k\right\|_2^2 \\
       \nonumber &\quad + \left (\nabla h\left (\zbmbar^k\right ) - \nabla h\left (\zbm^k\right )\right )^\Tsf  \left (\xbm^k - \zbm^k\right )  + \frac{1}{\gamma}\left (\zbmbar^k - \zbm^k\right )^\Tsf \left (\xbm^k - \zbm^k\right )\\
     & > g\left (\xbm^k\right ) + h\left (\xbm^k\right ) -\left\|\nabla h\left (\zbmbar^k\right ) - \nabla h\left (\zbm^k\right )\right\|_2\left\|\xbm^k - \zbm^k\right\|_2 -\frac{1}{\gamma}\left\|\zbmbar^k - \zbm^k\right\|_2 \left\|\xbm^k - \zbm^k\right\|_2, 
    \end{align}
    where we added and subtracted the term $\nabla h(\zbm^k)^\Tsf  (\xbm^k - \zbm^k)$ in the third line and used Cauchy-Schwarz inequality in the last line. 

\medskip\noindent
    Now from the Lagrange multiplier update rule $\sbm^k = \sbm^{k-1} +\xbm ^k -\zbm^k$, equations~\eqref{Eq:boundonSinexact} and \eqref{Eq:2Rbound}, we obtain 
    \begin{equation}\label{Eq:BoundonDualVar}
        \left\|\xbm^k -\zbm^k \right\|_2 =\left\|\sbm^k - \sbm^{k-1}\right\|_2 \leq \gamma L \left\|\zbm^k - \zbm^{k-1}\right\|_2 \leq 2\gamma L R.
    \end{equation}
    By using the bound on the distance of target and mismatched denoisers in~Assumption~\ref{As:InexactDistance}, Lipschitz continuity of $\nabla h$ in Assumption~\ref{As:LipschitzPrior}, equations~\eqref{Eq:boundlagnew} and \eqref{Eq:BoundonDualVar}, we get
    \begin{equation*}
                \phi\left (\xbm^k, \zbm^k, \sbm^k\right )  >g\left (\xbm^k\right ) + h\left (\xbm^k\right ) - 2\left (1+\gamma L\right ) R L \delta_k.
    \end{equation*}
    From the fact that both functions $g$ and $h$ are bounded from below in Assumption \ref{As:BoundedFromBelow} and the fact that $\gamma$, $\delta_k$, $R$, and $L$ are finite constants, we conclude that the augmented Lagrangian is bounded from below. This is equivalent to the  existence of $\phi^\ast = \phi(\xbm^\ast, \zbm^\ast, \sbm^\ast) > -\infty$ such that we have almost surely $\phi^\ast \leq \phi(\xbm^k, \zbm^k, \sbm^k)$, for all $k \geq 1$.
\end{proof}

%%%%%%%%%%%%%%%%%%%%%%%%%%%%%%%%%%%%%%%%%%%%%%%%%%%%%%%%%%%%%

\section{Proof of Theorem 2}\label{Sup:Sec:theorem2}
\begin{suptheorem}
    Run PnP-ADMM with the MMSE denoiser for $t \geq 1$ iterations under Assumptions~\ref{As:NonDegenerate}-\ref{As:LipschitzPrior} with the penalty parameter  $0<\gamma\leq 1/(4L)$. Then, we have 
    \begin{equation*}
       \min_{1\leq k\leq t} \left\|\nabla f\left (\xbm^k\right )\right\|_2^2 \leq  \frac{1}{t} \sum_{k=1}^t\left\|\nabla f\left (\xbm^k\right )\right\|^2_2 
    \leq \frac{C}{t} \left (\phi\left (\xbm^0, \zbm^0, \sbm^0\right )-\phi^\ast\right ) , 
    \end{equation*}
where $C>0$ is an iteration independent constant.
\end{suptheorem}

\begin{proof}
    Note that for PnP-ADMM with the MMSE denoiser, Lemma \ref{lem:LagDecExact} states
    \begin{equation}
       \left  \|\zbm^k -\zbm^{k-1} \right\|_2^2  \leq \frac{2\gamma}{1- \gamma L - 2 \gamma^2 L^2}\left ( \phi\left (\xbm^{k-1}, \zbm^{k-1} , \sbm^{k-1}\right ) - \phi\left (\xbm^k , \zbm^k , \sbm^k\right )\right ).
    \end{equation}
   By averaging over $t \geq 1$ iterations and using the fact that the augmented Lagrangian is bounded from below in Lemma~\ref{Lem:BoundBelowExact}, we obtain
    \begin{align}\label{Eq:ConvIterStoExact}
\nonumber\frac{1}{t} \sum_{k=1}^t\left\|\zbm^k - \zbm^{k-1}\right\|\  &\leq \frac{B}{t}\left ( \phi\left (\xbm^0, \zbm^0 , \sbm^0\right ) - \phi\left (\xbm^t , \zbm^t , \sbm^t\right )\right ) \\
&\leq \frac{B}{t}\left ( \phi\left (\xbm^0, \zbm^0 , \sbm^0\right ) - \phi^* \right ), 
    \end{align}
where $B \defn 2\gamma/(1 - \gamma L - 2 \gamma^2 L^2)$. From the optimality conditions for the MMSE denoiser $\Dsf_\sigma$ and $\xbm^k = \prox_{\gamma g}(\zbm^{k-1} - \sbm^{k-1})$, we have
\begin{equation}\label{Eq:optimalitCon}
  \nabla g\left (\xbm^k \right ) +  \frac{1}{\gamma}\left ( \xbm^k + \sbm^{k-1} - \zbm^{k-1} \right ) = \bm{0} \quad \text{and}\quad \nabla h\left (\zbm^k \right ) + \frac{1}{\gamma}\left (\zbm^{k} - \xbm^k - \sbm^{k-1}\right ) =\bm{0}.
\end{equation}
By using the $L$-Lipschitz continuity of $\nabla h$ and the Lagrange multiplier update rule in the form of $\sbm^ k = \sbm^{k-1} + \zbm^k - \xbm^k$, we can write 
\begin{align*}\label{Eq:boundonXZ}
  \left \|\nabla h\left (\xbm^k\right ) - \nabla h\left (\zbm^k\right )\right\|_2  &\leq L \left\| \xbm^k -  \zbm^k\right\|_2 
   = L \left\|\sbm^k -\sbm^{k-1}\right\|_2\\
   \nonumber & = \gamma L\left \|\nabla h\left (\zbm^k\right ) - \nabla h\left (\zbm^{k-1}\right ) \right\|_2 \\
    & \leq \gamma L^2 \left\|\zbm^k - \zbm^{k-1}\right\|_2.
\end{align*}
By using this equation and equation~\eqref{Eq:optimalitCon}, we have for the objective function in~\eqref{Eq:OptimInverseProb} 
\begin{align}
   \nonumber \left\|\nabla f\left (\xbm^k\right )\right\|_2 
   \nonumber&= \left\|\nabla g\left (\xbm^k\right ) + \nabla h\left (\xbm^k\right )\right\|_2 \\
   \nonumber&= \left \|\nabla g\left (\xbm^k\right )  +  \frac{1}{\gamma}\left ( \xbm^k + \sbm^{k-1} - \zbm^{k-1} \right ) + \nabla h\left (\xbm^k\right ) + \frac{1}{\gamma} \left (\zbm^{k-1}- \sbm^{k-1} -\xbm^k \right) \right\|_2 \\
   \nonumber& = \|\nabla h\left (\zbm^k\right )  + \frac{1}{\gamma} \left (\zbm^{k} - \xbm^k - \sbm^{k-1}\right  )  + \nabla h\left (\xbm^k\right )- \nabla h\left (\zbm^k\right ) + \frac{1}{\gamma}  \left (\zbm^{k-1}- \zbm^k \right )\|_2\\
   \nonumber&\leq  \left\| \nabla h\left (\xbm^k\right )- \nabla h\left (\zbm^k\right )\right \|_2 + \frac{1}{\gamma} \left\|\zbm^k- \zbm^{k-1}\right\|_2\\
   \nonumber & \leq \left ( \frac{1}{\gamma} + \gamma L^2\right ) \left\|\zbm^k- \zbm^{k-1}\right\|_2
\end{align}
where we used triangle inequality in the first inequality.
By squaring both sides, averaging over $t \geq 1$ iterations, and usi equation~\eqref{Eq:ConvIterStoExact}, we get the desired result
\begin{equation*}
      \min_{1\leq k\leq t} \left\|\nabla f\left (\xbm^k\right ) \right\|_2^2 \leq  \frac{1}{t} \sum_{k=1}^t\left\|\nabla f\left (\xbm^k\right )\right\|^2_2  \leq \frac{C}{t}\left ( \phi\left (\xbm^0, \zbm^0 , \sbm^0\right ) - \phi^*\right)
\end{equation*}
where $C \defn B(1+\gamma^2 L^2)/\gamma^2$.
\end{proof}

%%%%%%%%%%%%%%%%%%%%%%%%%%%%%%%%%%%%%%%%%%%%%%%%%%%%%%%%%%%%%
\section{Useful results for Theorem 2}\label{Sup:Sec:lemmas2}
\begin{lemma}\label{lem:LagDecExact}
    Assume that Assumptions~\ref{As:NonDegenerate}-\ref{As:LipschitzPrior} hold and let the sequence $\{\xbm^k , \zbm^k , \sbm^k\}$ be generated via iterations of PnP-ADMM with the MMSE denoiser using the penalty parameter $ 0<\gamma<1/(4L)$.
Then for the augmented Lagrangian defined in~\ref{Eq:AugLagrangian}, we have that
\begin{equation*}
    \label{Eq:LgDecreasing}
 \phi\left (\xbm^k , \zbm^k , \sbm^k\right ) \leq \phi\left (\xbm^{k-1}, \zbm^{k-1} , \sbm^{k-1}\right ) -  \left (\frac{1- \gamma L - 2 \gamma^2 L^2}{2\gamma}\right ) \left\|\zbm^k -\zbm^{k-1} \right\|_2^2.
\end{equation*} 
\end{lemma}

\begin{proof}
From the smoothness of  $h$  for any $\zbm^k \in \Im(\Dsf_\sigma)$, the optimality condition for the MMSE denoiser, and the Lagrange multiplier update rule in the form of $\sbm^k = \sbm^{k-1} +\xbm^k - \zbm^k$, we have
\begin{equation*}
    \nabla h\left (\zbm^k\right ) = \frac{1}{\gamma} \left (\xbm^k + \sbm^{k-1} - \zbm^k \right ) = \frac{1}{\gamma} \sbm^k .
\end{equation*}
From this equality and the definition of the augmented Lagrangian in~\eqref{Eq:AugLagrangian}, we have
\begin{align}\label{Eq:LagIneqExact1}
    \nonumber\phi\left (\xbm^k , \zbm^k , \sbm^k\right ) - \phi\left (\xbm^k , \zbm^k , \sbm^{k-1}\right ) 
    \nonumber&= \frac{1}{\gamma} \left (\sbm^k - \sbm^{k-1}\right )^\Tsf \left (\xbm^k - \zbm^k\right )\\
    & \nonumber = \frac{1}{\gamma} \left \|\sbm^k - \sbm^{k-1}\right\|_2^2 = \gamma\left\|\nabla h\left (\zbm^k \right ) - \nabla h\left (\zbm^{k-1}\right )\right\|^2_2\\
    & \leq \gamma L^2 \left\|\zbm^k - \zbm^{k-1}\right\|_2^2,  
\end{align}
where in the last line we used $L$-Lipschitz continuity of $\nabla h$ in Assumption~\ref{As:LipschitzPrior}. Additionally, we have
\begin{align*}
    h\left (\zbm^k\right ) - h\left (\zbm^{k-1}\right ) 
    \nonumber &\leq \nabla h\left (\zbm^k\right ) ^\Tsf \left (\zbm^k - \zbm^{k-1}\right ) + \frac{L}{2}\left\|\zbm^k - \zbm^{k-1}\right\|_2^2\\
     & = \frac{1}{\gamma} \left (\sbm^k\right )^\Tsf \left (\zbm^k - \zbm^{k-1}\right ) + \frac{L}{2}\left\|\zbm^k - \zbm^{k-1}\right\|_2^2.
\end{align*}
Now by using this equation and the definition of the augmented Lagrangian, we have 
\begin{align}\label{Eq:LagIneqExact2}
    \nonumber \phi\left (\xbm^k , \zbm^k , \sbm^{k-1}\right )  &- \phi\left (\xbm^k , \zbm^{k-1} , \sbm^{k-1}\right ) =  h\left (\zbm^k\right ) - h\left (\zbm^{k-1}\right )+ \frac{1}{\gamma} \left ({\sbm^{k-1}}\right )^\Tsf \left (\zbm^{k-1} - \zbm^k\right )\\
    \nonumber& \quad + \frac{1}{2\gamma}\left\|\xbm^k - \zbm^k\right \|_2^2 -\frac{1}{2\gamma}\left\|\xbm^k - \zbm^{k-1} \right\|_2^2 \\
    \nonumber &=  h\left (\zbm^k\right ) - h\left (\zbm^{k-1}\right ) + \frac{1}{\gamma} \left ({\sbm^{k-1}}\right )^\Tsf \left (\zbm^{k-1} - \zbm^k\right )\\
    \nonumber &\quad  + \frac{1}{2\gamma}\left (2\xbm^k - \zbm^k - \zbm^{k-1}\right )^\Tsf \left (\zbm^{k-1} - \zbm^k\right )\\
    \nonumber& =  h\left (\zbm^k\right ) - h\left (\zbm^{k-1}\right )+ \frac{1}{\gamma} \left ({\sbm^{k-1}}\right )^\Tsf \left (\zbm^{k-1} - \zbm^k\right ) \\
    \nonumber &\quad+ \frac{1}{\gamma}\left (\sbm^k - \sbm^{k-1}\right )^\Tsf \left (\zbm^{k-1} - \zbm^k\right ) - \frac{1}{2\gamma}\left\|\zbm^k - \zbm^{k-1}\right\|^2_2\\
    \nonumber& \leq \frac{1}{\gamma} \left (\sbm^k\right )^\Tsf \left (\zbm^k - \zbm^{k-1}\right ) + \frac{L}{2}\left\|\zbm^k - \zbm^{k-1}\right\|_2^2 + \frac{1}{\gamma} \left ({\sbm^{k-1}}\right )^\Tsf \left (\zbm^{k-1} - \zbm^k\right ) \\
    \nonumber &\quad \quad + \frac{1}{\gamma}\left (\sbm^k - \sbm^{k-1}\right )^\Tsf \left (\zbm^{k-1} - \zbm^k\right ) - \frac{1}{2\gamma}\left\|\zbm^k - \zbm^{k-1}\right\|^2_2\\
    & \leq  - \left (\frac{1-\gamma L}{2\gamma}\right )\left\|\zbm^k - \zbm^{k-1}\right\|^2_2. 
\end{align}
Note that from $\xbm^k = \prox_{\gamma g} (\zbm^{k-1} - \sbm^{k-1})$, we have 
 \begin{align*}
    \frac{1}{2\gamma}\left\|\xbm^{k} -\zbm^{k-1} + \sbm^{k-1}\right\|_2^2 + g\left (\xbm^k \right ) &=  \min_{\xbm \in \R^n} \{\frac{1}{2\gamma}\left\|\xbm -\zbm^{k-1} + \sbm^{k-1}\right\|_2^2 + g\left (\xbm\right )\}\\
    & \leq \frac{1}{2\gamma}\left\|\xbm^{k-1} -\zbm^{k-1} + \sbm^{k-1}\right\|_2^2 + g\left (\xbm^{k-1} \right ), 
 \end{align*}
which implies that 
\begin{equation}\label{Eq:LagIneqExact3}
    \phi\left (\xbm^k, \zbm^{k-1}, \sbm^{k-1}\right )\leq \phi\left (\xbm^{k-1}, \zbm^{k-1}, \sbm^{k-1}\right ).
\end{equation}
Now by combining the results from equations~\eqref{Eq:LagIneqExact1},~\eqref{Eq:LagIneqExact2} and~\eqref{Eq:LagIneqExact3}, we have 
\begin{equation*}
    \phi\left (\xbm^k , \zbm^k , \sbm^k\right ) \leq \phi\left (\xbm^{k-1}, \zbm^{k-1} , \sbm^{k-1}\right ) -  \left (\frac{1- \gamma L - 2 \gamma^2 L^2}{2\gamma}\right ) \left\|\zbm^k -\zbm^{k-1} \right\|_2^2.
\end{equation*}
\end{proof}

\begin{lemma}\label{Lem:BoundBelowExact}
Assume that Assumptions~\ref{As:NonDegenerate}-\ref{As:LipschitzPrior} hold and let the sequence $\{\xbm^k, \zbm^k , \sbm^k\}$ be generated via PnP-ADMM with the MMSE denoiser using the penalty parameter $ 0<\gamma<1/(4L)$. Then the augmented Lagrangian $\phi$ defined in~\eqref{Eq:AugLagrangian} is bounded from below
\begin{equation*}
    \inf_{\scriptscriptstyle{k\geq1}}\phi\left ( \xbm^k, \zbm^k, \sbm^k\right ) > -\infty. 
\end{equation*}
\end{lemma}
\begin{proof}
From the smoothness of  $h$  for any $\zbm^k \in \Im(\Dsf_\sigma)$, the optimality condition for the MMSE denoiser, and the Lagrange multiplier update rule in the form of $\sbm^k = \sbm^{k-1} +\xbm^k - \zbm^k$, we have
\begin{equation}\label{Eq:optimofHexact}
    \nabla h\left(\zbm^k\right ) = \frac{1}{\gamma} \left(\xbm^k + \sbm^{k-1} - \zbm^k \right ) = \frac{1}{\gamma} \sbm^k .
\end{equation}
    By using  the $L$-Lipschitz continuity of $\nabla h$ in Assumption~\ref{As:LipschitzPrior}, we have that
    \begin{equation}\label{Eq:LipofGradhexact}
         h\left(\xbm^k\right ) \leq h\left(\zbm^k\right ) + \nabla h\left(\zbm^k\right )^\Tsf \left(\xbm^k - \zbm^k\right  ) + \frac{L}{2}\left\|\xbm^k - \zbm^k\right\|_2^2.
    \end{equation}
    From equations~\eqref{Eq:optimofHexact},~\eqref{Eq:LipofGradhexact} and the fact that $\gamma L < 1$, we have
       \begin{align}\label{Eq:augmentedLagrang1}
        \nonumber\phi\left(\xbm^k, \zbm^k, \sbm^k\right  ) &= g\left(\xbm^k\right ) + h\left(\zbm^k\right ) + \frac{1}{\gamma}\left(\sbm^k\right )^\Tsf \left(\xbm^k - \zbm^k\right ) + \frac{1}{2\gamma} \left\|\xbm^k - \zbm^k\right\|_2^2\\
        \nonumber&  > g\left(\xbm^k\right ) + h\left(\zbm^k\right ) + \nabla h\left(\zbm^k\right )^\Tsf  \left(\xbm^k - \zbm^k\right ) + \frac{L}{2} \left\|\xbm^k - \zbm^k\right\|_2^2\\
        \nonumber &  > g\left(\xbm^k\right ) +h\left(\xbm^k\right ).
    \end{align}
    Note that since both functions $g$ and $h$ are bounded from below from Assumption \ref{As:BoundedFromBelow} , we conclude that the augmented   Lagrangian is bounded from below. This implies that there exists $\phi^\ast = \phi(\xbm^\ast, \zbm^\ast, \sbm^\ast) > -\infty$ such that we have  $\phi^\ast \leq \phi(\xbm^k, \zbm^k, \sbm^k)$, for all $k \geq 1$.  
\end{proof}

\section{Background material}\label{sup:sec:addtionaltheory}
\subsection{MMSE denoising as proximal operator}
\label{ssec:MMSEdenoiseProx}

The connection between MMSE estimation and regularized inversion was established by Gribonval in~\cite{Gribonval2011}, and this relationship has been explored in various contexts~\cite{Gribonval.Machart2013, Kazerouni.etal2013, Gribonval.Nikolova2021, gan2023block}. This connection was formally linked to Plug-and-Play (PnP) methods in~\cite{Xu.etal2020}, resulting in a novel interpretation of MMSE denoisers within the framework of PnP. In this section, we investigate the fundamental argument that bridges MMSE denoising and proximal operators.

\medskip\noindent
The MMSE estimator for the following AWGN denoising problem 
\begin{equation}\label{Eq:noisemodelApp}
    \ubm = \xbm +\ebm \quad \text{with} \quad \xbm \sim \widehat{p}_\xbm, \quad \ebm \sim \Ncal(0, \sigma^2 \Ibm),
\end{equation}
is expressed as 

\begin{equation}
\label{Eq:MMSEDenoiserApp}
\Dsf_{\sigma}(\ubm) \defn \E[\xbm | \ubm] = \int_{\R^{n}} \xbm p_{\xbm | \ubm} (\xbm | \ubm) \d \xbm.
\end{equation}

\medskip\noindent
From \emph{Tweedie's formula}, we can express the estimator~\eqref{Eq:MMSEDenoiserApp} as 
\begin{equation}
\label{Eq:Tweedie}
\Dsf_{\sigma}(\ubm) = \ubm - \sigma^2 \nabla h_{\sigma}(\ubm) \quad\text{with}\quad h_{\sigma}(\ubm) = -\log(p_{\ubm}(\ubm)),
\end{equation}
which is derived by differentiating~\eqref{Eq:MMSEDenoiserApp} using the expression for the probability distribution given by
\begin{equation}
\label{Eq:ConvRel}
p_{\ubm}(\ubm) = (p_{\xbm} \ast \phi_{\sigma})(\ubm) = \int_{\R^{n}} \phi_{\sigma}(\ubm-\xbm)p_{\xbm}(\xbm) \d \xbm,
\end{equation}
where
$$\phi_{\sigma}(\xbm) \defn \frac{1}{(2\pi\sigma^2)^{\frac{n}{2}}} \exp\left(-\frac{\|\xbm\|^2}{2\sigma^2}\right).$$
Since $\phi_{\sigma}$ is infinitely differentiable, the same applies to $p_{\ubm}$ and $\Dsf_{\sigma}$.  As demonstrated in Lemma 2 of~\cite{Gribonval2011}, the Jacobian of $\Dsf_{\sigma}$ is positive definite:
\begin{equation}
\label{Eq:Jacobian}
\Jsf\Dsf_{\sigma}(\ubm) = \Ibf - {\sigma}^2 \Hsf h_{\sigma}(\ubm) \succ 0, \quad \ubm \in \R^{n},
\end{equation}
where $\Hsf h_{\sigma}$ represents the Hessian matrix of the function $h_{\sigma}$. Additionally, Assumption~\ref{As:NonDegenerate} implies that $\Dsf_{\sigma}$ is a \emph{one-to-one} mapping from $\R^{n}$ to $\Im(\Dsf_{\sigma})$. This implies that ${(\Dsf_{\sigma})^{-1}: \Im(\Dsf_{\sigma}) \rightarrow \R^{n}}$ is well defined and infinitely differentiable over $\Im(\Dsf_{\sigma})$, as outlined in Lemma 1 of~\cite{Gribonval2011}. Consequently, this indicates that the regularizer $h$ in~\eqref{Eq:ExpReg} is also infinitely differentiable for any $\xbm \in \Im(\Dsf_{\sigma})$.

\medskip\noindent
We will now establish that
\begin{align}
\label{Eq:DenoiserIsProx}
\Dsf_{\sigma}(\ubm) &= \prox_{\gamma h}(\ubm) = \argmin_{\xbm \in \R^{n}}\left\{\frac{1}{2}\|\xbm-\ubm\|^2 + \gamma h(\xbm)\right\}\nonumber
\end{align}
where $h$ is a (possibly nonconvex) function defined in~\eqref{Eq:ExpReg}. Our objective is to demonstrate that $\ybm^\ast = \ubm$ is the unique stationary point and global minimizer of
$$
\varphi(\ybm) \defn \frac{1}{2}\|\Dsf_{\sigma}(\ybm)-\ubm\|^2 + \gamma h(\Dsf_{\sigma}(\ybm)), \quad \ybm \in \R^{n}.
$$
By using the definition of $h$ in~\eqref{Eq:ExpReg} and the Tweedie's formula~\eqref{Eq:Tweedie}, we obtain
\begin{align*}
\varphi(\ybm) 
= \frac{1}{2}\|\Dsf_{\sigma}^\ast(\ybm)-\ubm\|^2 - \frac{\sigma^4}{2}\|\nabla h_{\sigma}(\ybm)\|^2 + \sigma^2h_{\sigma}(\ybm).
\end{align*}
The gradient of $\varphi$ is then given by
\begin{align*}
&\nabla \varphi(\ybm) 
= [\Jsf\Dsf_{\sigma}(\ybm)](\Dsf_{\sigma}(\ybm)-\ubm) + \sigma^2 [\Ibf - \sigma^2 \Hsf h_{\sigma}(\ybm)]\nabla h_{\sigma}(\ybm) = [\Jsf \Dsf_{\sigma}(\ybm)](\ybm-\ubm),
\end{align*}
where we used~\eqref{Eq:Jacobian} in the second line and~\eqref{Eq:Tweedie} in the third line. Consider a scalar function  ${q(\nu) = \varphi(\ubm+\nu\ybm)}$ and its derivative
$$q'(\nu) = \nabla \varphi(\ubm+\nu\ybm)^\Tsf\ybm = \nu \ybm^\Tsf [\Jsf\Dsf_{\sigma}^\ast(\ubm+\nu\ybm)]\ybm.$$
The positive definiteness of the Jacobian~\eqref{Eq:Jacobian} implies that $q'(\nu) < 0$ and $q'(\nu) > 0$ for $\nu < 0$ and $\nu > 0$. Thus, $\nu = 0$ is the global minimizer of $q$. Since $\ybm \in \R^{n}$ is arbitrary, we can conclude that $\varphi$ has no stationary point other than $\ybm^\ast = \ubm$, and that $\varphi(\ubm) < \varphi(\ybm)$ for any $\ybm \neq \ubm$~\cite{Xu.etal2020}.

\section{On the assumptions of Theorem~\ref{Thm:MainThm}} \label{sup:sec:assumpList}

In this section, we present the list of assumptions required for Theorems~\ref{Thm:MainThm}. Assumptions required for Theorems are typically employed when using MMSE estimators as PnP priors, engaging in nonconvex optimization, or dealing with mismatched/inexact PnP priors. 

\textbf{Assumptions of Theorem~\ref{Thm:MainThm}:}
\begin{itemize}
    \item \emph{Prior distributions $p_{\xbm}$ and  $\widehat{p}_{\xbm}$, denoted as target and mismatched distributions are   non-degenerate over $\R^n$. } 
    
    As discussed in Section~\ref{subsec:theory}, this assumption is commonly adopted  to establish a relation between regularized inversion and MMSE estimation~\cite{Gribonval2011, Gribonval.Machart2013, Kazerouni.etal2013}. The MMSE estimators have been previously used as priors in PnP methods~\cite{Xu.etal2020, gan2023block, Laumont.etal2022}.
    
    \item \emph{ Function $g$ (data-fidelity term) is continuously differentiable.}
    
    This assumption is an standard assumption commonly adopted in nonconvex optimization, specifically in the context of inverse problems~\cite{li2018simple, jiang2019structured,yashtini2021multi}. It is worth noting that the majority of well-established data-fidelity terms for image restoration tasks fall under the umbrella of this assumption. Importantly, this framework does not necessitate the convexity of data-fidelity terms, making it versatile for handling non-linear measurement models. Furthermore, our result can be extended to a non-differentiable data-fidelity term $g$ by using subdifferentials, making it applicable to applications like phase retrieval~\cite{Metzler.etal2018}.
    
    \item \emph{The explicit data-fidelity term $g$ and the implicit regularizer $h$ are bounded from below.} 
    
    This assumption is standard in optimization and ensures that the optimization problem is well-posed and has a meaning full solution. This Assumption is commonly adopted in optimization algorithms~\cite{yashtini2021multi, hurault2022proximal,Hurault.etal2022,Xu.etal2020}.
    
    \item \emph{The denoisers $\Dsf_\sigma$ and $\Dhat_\sigma$ have the same range $\Im(\Dsf_\sigma)$.
    Additionally, functions $h$ and $\hine$ associated with $\Dsf_\sigma$ and $\Dhat_\sigma$, are continuously differentiable with $L$-Lipschitz continuous gradients over $\Im(\Dsf_\sigma).$}

    For the image denoisers that share the same architecture and employ the same loss function, it is reasonable to assume that their output range would be consistent, given that it aligns with the range of natural color images.
    Furthermore, due to the smoothness properties of both $\Dsf_\sigma^{-1}$ and $h_\sigma$ as described in equation~\ref{Eq:ExpReg}, it follows that the function $h$ is also smooth and continuously differentiable. A similar property holds for the function $\hine$ corresponding to the mismatched denoiser $\Dhat_\sigma$. Consequently, this assumption is a mild requirement, only necessitating that regularization functions have $L$-Lipschitz continuous gradients over their shared range. While the assumption of Lipschitz continuous gradients is a standard one in nonconvex optimization, it is typically enforced over the entire space $\R^n$, whereas here, we specifically enforce it over the range of the denoisers.~\cite{Hurault.etal2022, yashtini2021multi}. 

    \item \emph{The distance between the target and mismatched denoisers are bounded at each iteration of the algorithm.} 
    
    This assumption bounds the distance between the mismatched and target denoisers, which serves as a measure of the distribution shift. As the distributions used to train the mismatched denoisers diverge from the target distribution, we anticipate the bound on denoisers' distance will also increase. This assumption is a common one in the context of dealing with approximate, inexact, or mismatched priors~\cite{Laumont.etal2022, shoushtari2022deep, gan2023block}.

    \item \emph{The distance of sequence $(\zbm^k)$ given by the Algorithm~\ref{Alg:InexactPnPADMM} to stationary point $\zbm^*$ is bounded by a constant.}
    
    As depicted in Algorithm~\ref{Alg:InexactPnPADMM}, sequence $\zbm^k$ is the output of mismatched denoiser at each iteration. Since many denoisers have bounded range spaces, the existence of bound $R$ often holds. Specifically, this is true for such image denoisers whose output live within the bounded subset $[0,255]^n\subset \R^n$ or $[0,1]^n\subset \R^n$~\cite{Sun.etal2021, Sun.etal2019b}.
\end{itemize}

\section{Additional Technical Details}\label{sup:sec:additionalresults}
We present here some technical details and results that were not included in the main paper. In our quantitative comparisons of different priors, we utilized the Peak Signal-to-Noise Ratio (PSNR) metric, which is defined as follows:
\begin{equation*}
    PSNR(\xbmhat, \xbm) = 20\log_{10}\left ( \frac{1}{\|\xbmhat- \xbm\|_2}\right ), 
\end{equation*}
where $\xbm$ represents the ground truth and $\xbmhat$ denotes the estimated image. For our PnP-ADMM algorithm, we performed 15 iterations for all denoisers. The denoisers were trained using the DRUNet architecture~\cite{Zhang.etal2021dpir} with Mean Squared Error (MSE) loss, employing the Adam optimizer~\cite{Kingma.Ba2015} with a learning rate of $10^{-4}$. We incorporated a noise level map strength that decreases logarithmically from $\sigma_{\text{optim}}$ to $\sigma = 0.01$ over 15 iterations, where $\sigma_{\text{optim}}$ is fine-tuned for optimal performance for each test image and prior individually.
To prepare the training and testing images from datasets such as MetFaces~\cite{karras2020training}, AFHQ~\cite{choi2020stargan}, CelebA~\cite{liu2015faceattributes}, and RxRx1~\cite{sypetkowski2023rxrx1}, we randomly selected 1000 images and resized them to $256\times256$ slices. For the BreCaHAD~\cite{aksac2019brecahad} dataset, we cropped the images to $512\times512$ and subsequently resized them to $256\times256$ slices for both the training and testing datasets.

\medskip\noindent
Figure~\ref{fig:gt} shows the images that were used to generate measurements for super-resolution task. 
%%%%%%%%%%%%%%%%%%
\begin{figure}[!t]
    \centering
    \includegraphics[width=1\textwidth]{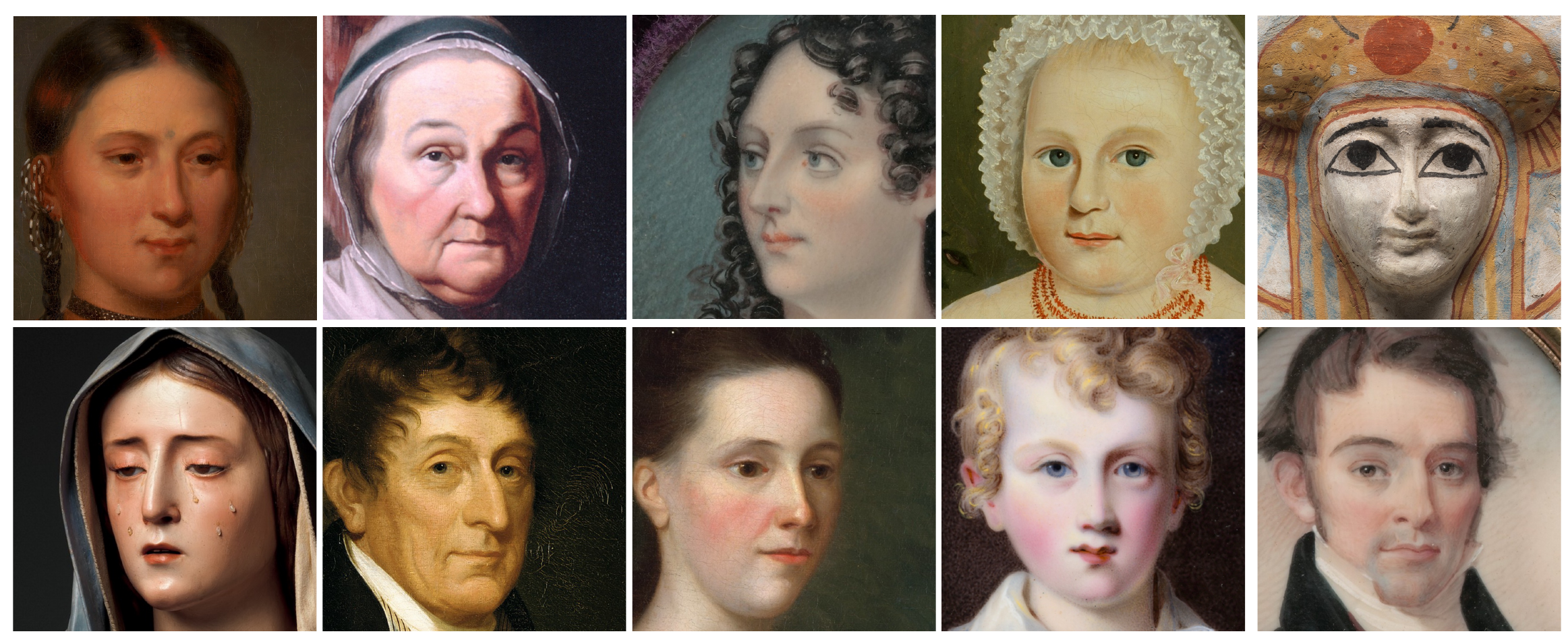}
    \caption{~\emph{\small Ground truth images from MetFaces dataset used for generating measurements.  }}
    \label{fig:gt}
\end{figure}
%%%%%%%%%%%%%%%%%%%

\section{Additional experiments}
\subsection{Super-resolution}
We present additional image super-resolution results for a more comprehensive understanding. 
Figure~\ref{fig:dnvsSR} illustrates the performance comparison of denoising and super-resolution using different priors. On the left side of Figure~\ref{fig:dnvsSR}, the denoising performance of target (trained on MetFaces), mismatched (trained on BreCaHAD), adapted, and retrained priors is displayed. Meanwhile, on the right side, the reconstruction performance of target, mismatched, and adapted priors is presented. Note the improvement achieved by using adapted priors in both denoising and super-resolution tasks.

%%%%%%%%%%%%%%%%%%
\begin{figure}[!t]
    \centering
    \includegraphics[width=1\textwidth]{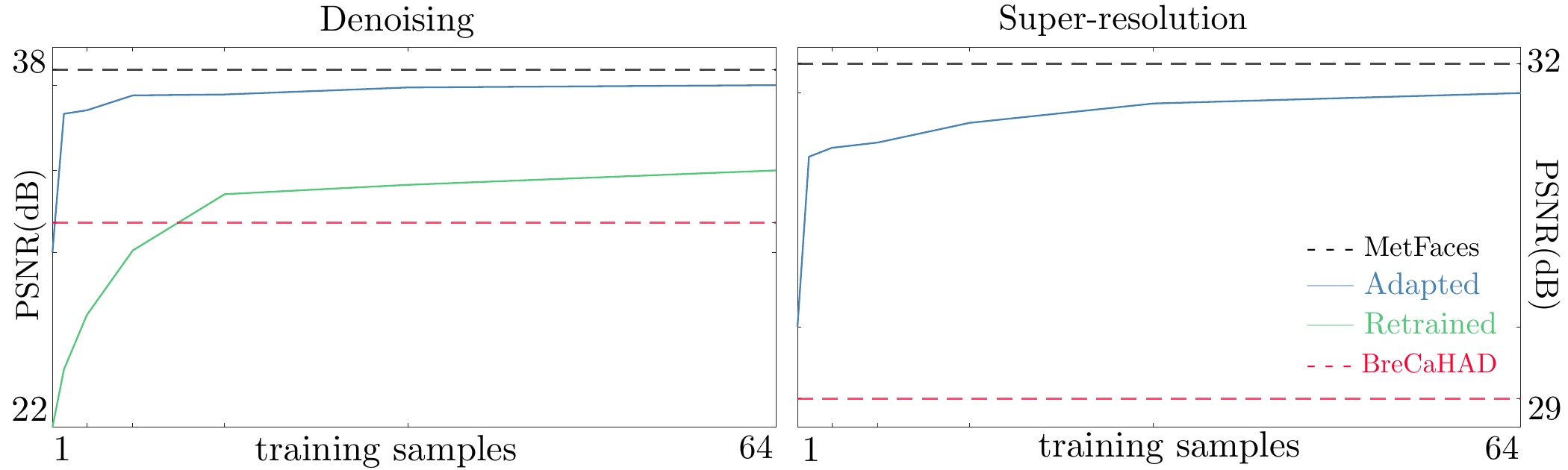}
    \caption{~\emph{\small
    The Left figure compares the empirical results of denoising for retrained and adapted priors vs. the number of training samples, as well as target (MetFaces) and mismatched (BreCaHAD) denoisers. The right figure compares PnP performance using target, mismatched, and adapted priors on super-resolution task. The results in both figures are reported for the test set from the MetFaces dataset,  averaged for scaling factor of $s=4$. It's worth highlighting the noticeable performance improvement of denoisers achieved through domain adaptation. Additionally, observe the relationship between PnP performance and  adapted denoiser performance.}}
    \label{fig:dnvsSR}
\end{figure}
%%%%%%%%%%%%%%%%%%%%%%%%%%%%%%%%%%%%

\subsection{Deblurring}
We present additional visual results for deblurring image restoration.  Figure~\ref{fig:deblurrmismatch} presents a visual comparison of a test image from the MetFaces dataset using the target denoiser and four different mismatched denoisers. The images are convolved with the indicated blur kernel and subjected to Gaussian noise with a noise level of $v = 0.01$. Note the suboptimal performance of mismatched priors in the deblurring task. As it is evident in Figure~\ref{fig:deblurrmismatch}, the discrepancy between the mismatched distributions directly affects the PnP performance. Figure~\ref{fig:deblurupdate} illustrates a visual comparison of adapted priors in the deblurring task. 

%%%%%%%%%%%%%%%%%%
\begin{figure}[!t]
    \centering
    \includegraphics[width=1\textwidth]{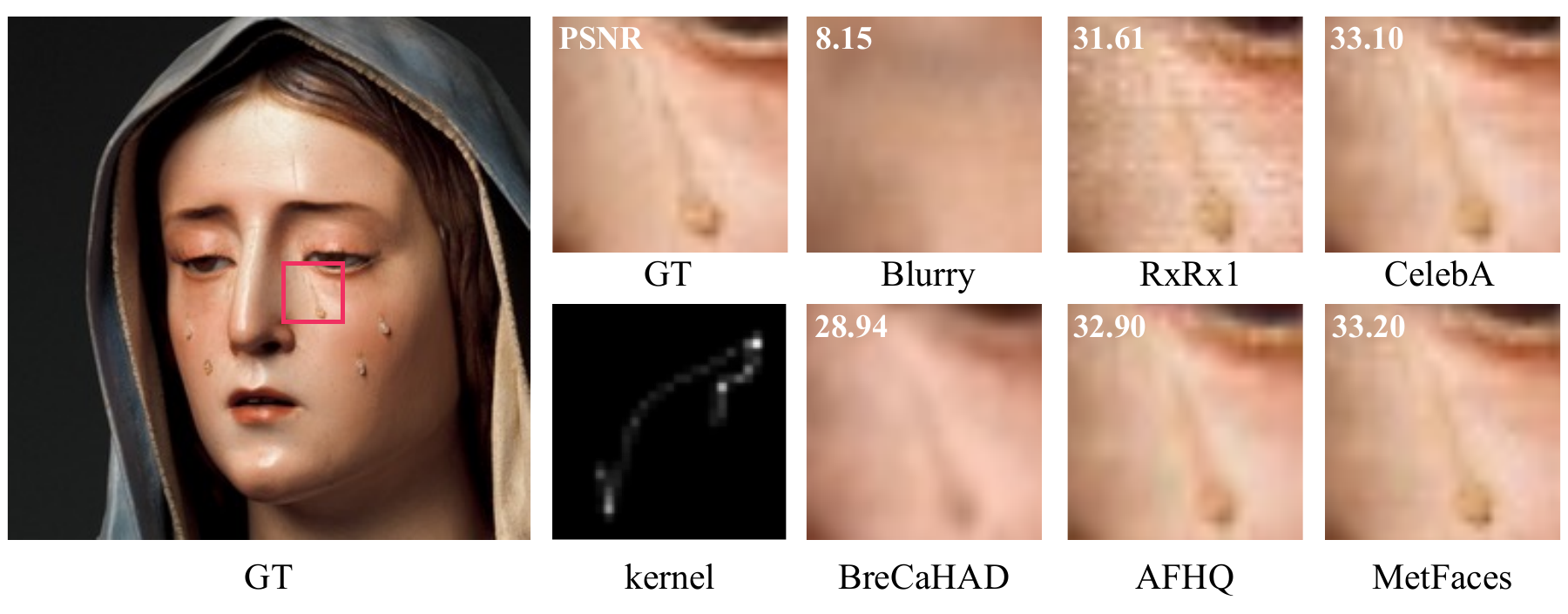}
    \caption{~\emph{\small Visual comparison of various mismatched denoisers for deblurring on an image from MetFaces dataset. The performance is reported in terms of PSNR (dB).
     The image is convolved with the indicated blur kernel and Gaussian noise with $v = 0.01$ is added. Note that regardless of the PnP image restoration task, the discrepancies in training distributions result in mismatched priors and suboptimal performance in PnP. }}
    \label{fig:deblurrmismatch}
\end{figure}
%%%%%%%%%%%%%%%%%%%

%%%%%%%%%%%%%%%%%%
\begin{figure}[t]
    \centering
    \includegraphics[width=1\textwidth]{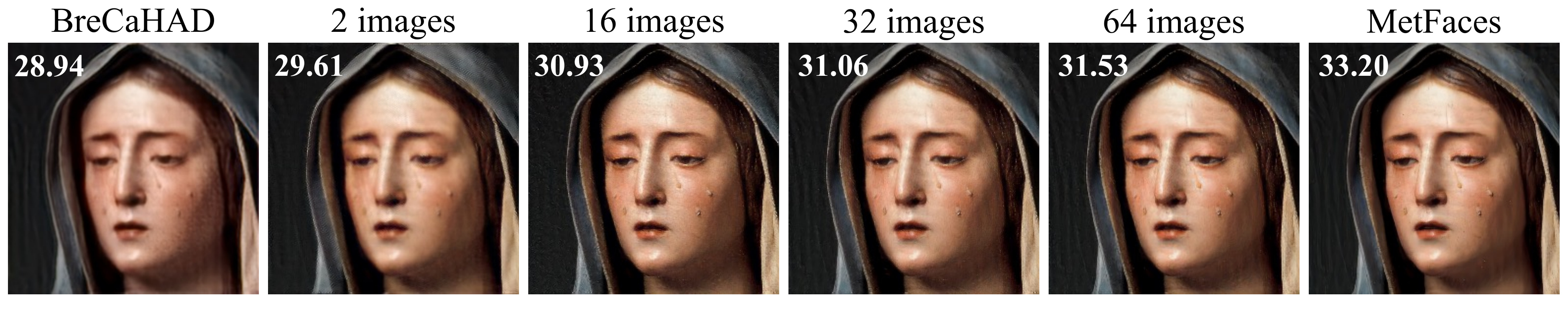}
    \caption{~\emph{\small Visual comparison of several adapted prior for image deblurring on a test image from MetFaces dataset. The performance is reported in terms of PSNR (dB). The experiment setting is similar to that of Figure~\ref{fig:deblurrmismatch}. Note how adapting the mismatched prior with a larger set of data from the target distribution results in a better performance in PnP. }}
    \label{fig:deblurupdate}
\end{figure}
%%%%%%%%%%%%%%%%%%%

\subsection{Various Distributions Experiment}

%%%%%%%%%%%%%%%%%%
\begin{figure}[t]
    \centering
    \includegraphics[width=1\textwidth]{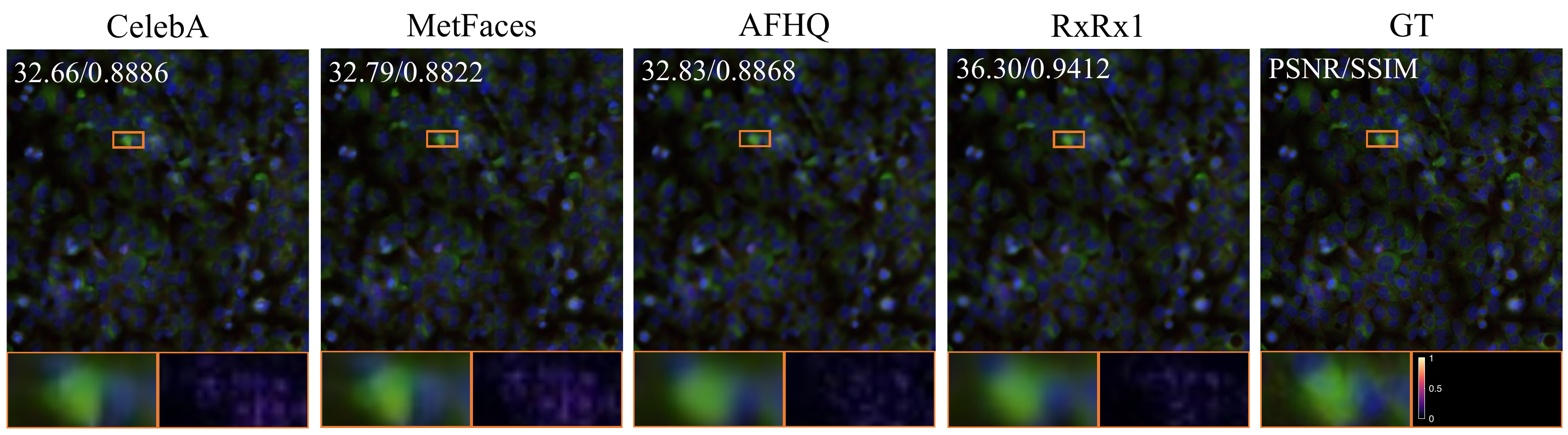}
    \caption{~\emph{\small Visual evaluation of several priors on the  image super-resolution task reported in terms of PSNR (dB) and SSIM for an image from RxRx1. Images are downsampled with the scale of $s=4$ and convolved with the indicated blur kernel in Figure~\ref{fig:SR_mismatched}. Note the influence of mismatched priors on the performance of PnP.  }}
    \label{fig:rxmm}
\end{figure}
%%%%%%%%%%%%%%%%%%%
We present additional visual results for mismatched priors and domain adaptation using various distributions for image super-resolution. In the following Figures, we  demonstrate the effect of mismatched priors and prior adaptation tested on an images from RxRx1~\cite{sypetkowski2023rxrx1} dataset.  Figure~\ref{fig:rxmm} presents a visual comparison for  PnP on super-resolution task using the target and three mismatched priors on an image from the RxRx1 test set. The images are convolved with the blur kernel indicted in Figure~\ref{fig:SR_mismatched}. Figure~\ref{fig:rxup} illustrates visual results for domain adaptation of mismatched prior trained on CelebA dataset and adapted to RxRx1 distribution. Note the improvement in PnP performance by using adapted priors. Also, note the relation between PnP performance and the number of samples from the target distribution used for adaptation.

%%%%%%%%%%%%%%%%%%
\begin{figure}[h]
    \centering
    \includegraphics[width=1\textwidth]{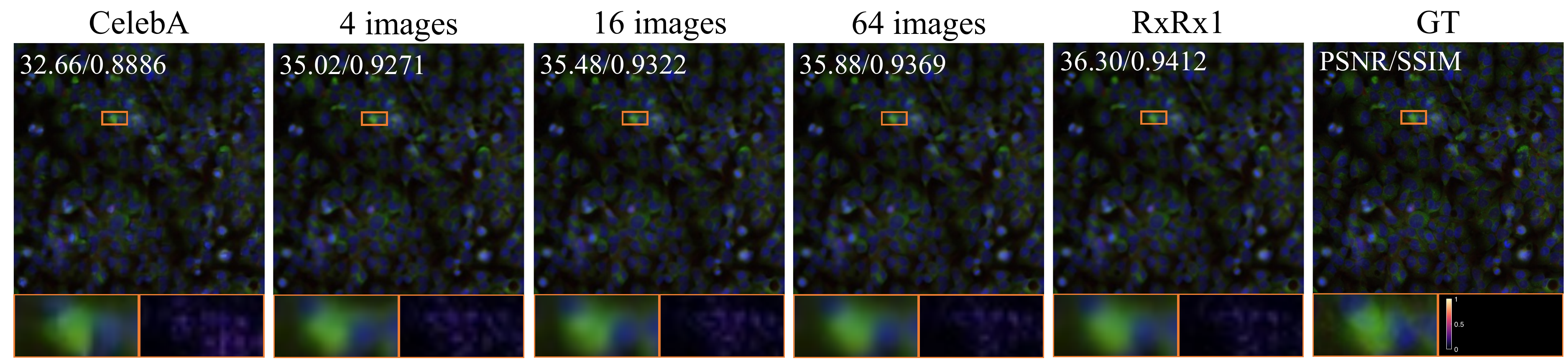}
    \caption{~\emph{\small Visual comparison of image super-resolution with target (RxRx1), mismatched (CelebA), and adapted priors on a test image from RxRx1. The images are downsampled by the scale of $s=4$. The performance is reported in terms of PSNR (dB) and SSIM. Note how the recovery performance increases by adaptation of mismatched priors to a larger set of images from the target distribution.}}
    \label{fig:rxup}
\end{figure}
%%%%%%%%%%%%%%%%%%%

\end{document}